\documentclass[a4paper]{article}
\usepackage[utf8x]{inputenc}
\usepackage[T1]{fontenc}
\usepackage{import}

\usepackage[a4paper,margin=1in]{geometry}
\usepackage[colorlinks=true,allcolors=blue,pagebackref]{hyperref}
\usepackage[round]{natbib}
\usepackage{algorithm,algorithmic}
\usepackage{amsfonts}
\usepackage{amsthm}
\usepackage{amsfonts}
\usepackage{hyperref}
\usepackage{enumitem}

\newtheorem{theorem}{Theorem}
\newtheorem{corollary}{Corollary}[theorem]
\newcommand{\cM}{\mathcal{M}}
\newtheorem{lemma}{Lemma}
\newcommand{\cbe}{C_{\textrm{BE}}}
\newcommand{\erf}{\textrm{erf}}
\newcommand{\cK}{\mathcal{K}}

\newcommand{\vectortwo}[2]{\left(\begin{smallmatrix}#1 \\ #2\end{smallmatrix}\right)}

\renewcommand{\v}{\mathbf{v}}
\newcommand{\interval}[2]{[#1,#2]}
\usepackage{framed}
\usepackage{amsthm,amsmath}
\usepackage{pdfpages}
\usepackage{lmodern}
\usepackage{caption}
\usepackage{subcaption}
\usepackage[capitalize]{cleveref}
\usepackage{thm-restate}
\newcommand{\z}{\mathbf{z}}

\newtheorem*{claim*}{Claim}
\usepackage{amsmath,amsfonts,graphicx}
\newcommand{\f}{\mathbf{f}}
\usepackage{environ}
\usepackage{comment}
\newcommand{\cS}{\mathcal{S}}
\newcommand{\cW}{\mathcal{W}}
\newcommand{\cA}{\mathcal{A}}

 \newcommand{\ignore}[1]{}
\renewcommand{\c}{\mathbf{c}}
\newcommand{\w}{\mathbf{w}}

\newcommand{\real}{\mathbb{R}}

\newcommand{\lboundary}{\xi_{0}}
\newcommand{\rboundary}{\xi_{1}}

\newcommand{\etwo}{\textbf{e}_2}
\newcommand{\thetatwo}{\theta_2}
\newcommand{\ERM}{F_S}

\newcommand{\wttwo}{w_2^{(t)}}

\newcommand{\eone}{\textbf{e}_1}
\newcommand{\wt}{\textbf{w}^{(t)}}
\newcommand\numberthis{\addtocounter{equation}{1}\tag{\theequation}}
\renewcommand{\sc}{m^A_\C}

\renewcommand{\u}{\mathbf{u}}

\newcommand{\K}{\mathcal{K}}

\crefname{property}{property}{properties}

\newcommand{\wsone}{w^S_1}
\newcommand{\wstwo}{w_2^S}

\newcommand{\ws}{\mathbf{w}_S}

\DeclareMathOperator*{\argmin}{arg\,min}

\newtheorem*{theorem*}{Theorem}
\newtheorem*{lemma*}{Lemma}

\newtheorem{claim}{Claim}[theorem]
\newtheorem{definition}{Definition}

\newcommand{\C}{\mathcal{C}}

\newcommand{\EE}{\mathop\mathbb{E}}

\usepackage{graphicx}
\usepackage{animate}
\newcommand{\wonet}[1]{w_1^{(#1)}}
\newcommand{\wtwot}[1]{w_2^{(#1)}}

\usepackage{wrapfig}
\usepackage{microtype}
\usepackage{graphicx}

\usepackage{booktabs} 

\usepackage[colorlinks=true,allcolors=blue,pagebackref]{hyperref}
\newcommand{\fullversion}[2]{#2}

\renewcommand{\smash}[1]{#1}

\title{Can Implicit Bias Explain Generalization? \\ Stochastic Convex Optimization as a Case Study}

\author{
Assaf Dauber\thanks{Tel Aviv University, Department of Electrical Engineering, \tt{ assafdauber@mail.tau.ac.il}.}
\and
Meir Feder\thanks{
Tel Aviv University, Department of Electrical Engineering, \tt{meir@tauex.tau.ac.il}.}
\and
Tomer Koren
\thanks{Tel Aviv University, School of Computer Science, and Google Research, \tt{tkoren@tauex.tau.ac.il}.}
\and
Roi Livni
\thanks{Tel Aviv University, Department of Electrical Engineering, \tt {rlivni@tauex.tau.ac.il}.}
}

\begin{document}
\maketitle
\begin{abstract}
The notion of implicit bias, or implicit regularization, has been suggested as a means to explain the surprising generalization ability of modern-days overparameterized learning algorithms. 
This notion refers to the tendency of the optimization algorithm towards a certain structured solution that often generalizes well. 
Recently, several papers have studied implicit regularization and were able to identify this phenomenon in various scenarios. 
We revisit this paradigm in arguably the simplest non-trivial setup, and study the implicit bias of Stochastic Gradient Descent (SGD) in the context of Stochastic Convex Optimization.
As a first step, we provide a simple construction that rules out the existence of a \emph{distribution-independent} implicit regularizer that governs the generalization ability of SGD.
We then demonstrate a learning problem that rules out a very general class of \emph{distribution-dependent} implicit regularizers from explaining generalization, which includes strongly convex regularizers as well as non-degenerate norm-based regularizations.
Certain aspects of our constructions point out to significant difficulties in providing a comprehensive explanation of an algorithm's generalization performance by solely arguing about its implicit regularization properties.
\end{abstract}

\section{Introduction}
One of the great mysteries of contemporary machine learning is the impressive success of \emph{unregularized} and \emph{overparameterized} learning algorithms. 
In detail, current machine learning practice is to train models with far more parameters than samples and let the algorithm \emph{fit} the data, oftentimes without any type of regularization. 
In fact, these algorithms are so overcapacitated that they can even memorize and fit random data. Yet, when trained on real-life data, these algorithms show remarkable performance in generalizing to unseen samples~\citep{NeyshaburTS14,ZhangBHRV17}.

This phenomenon is often attributed to what is described as the \emph{implicit-regularization} of an algorithm~\citep{NeyshaburTS14}. 
Implicit regularization roughly refers to the learner's preference to implicitly choosing certain structured solutions \emph{as if} some explicit regularization term appeared in its objective. 
As a canonical example, in linear optimization one can show that various forms of gradient descent, an apriori unregularized algorithm, behaves identically as regularized risk minimization penalized with the squared Euclidean norm on the parameters~\citep{shalev2011online}.

Understanding implicit regularization poses several interesting challenges. 
For example: how can we find the implicit bias of a given learning algorithm? 
what is the rate of convergence towards the biased solution? 
how (and if) does it govern the generalization of an algorithm? 
and, when and what types of regularizations can account for and explain the generalization in modern-days machine learning?

Towards answering these questions we revisit a fundamental setting that was extensively studied in recent years: Stochastic Convex Optimization (SCO), focusing on the SGD optimization algorithm.
In contrast to most previous work, we do not attempt to identify the implicit bias in specific problems. 
Instead, we study these questions in the general case, and we construct examples which rule out the existence of potential regularizers in general.
To some extent, these constructions demonstrate a behavior that might seem counter-intuitive or contradictory to the implicit-bias point of view.

Besides being a well-studied and well-understood model for learning, an important trait of SCO which makes it suitable for our investigation is that learning cannot in general be performed by naive \emph{Empirical Risk Minimization} (ERM). In detail, the work of \citet{shalev2009stochastic} showed the existence of SCO instances where naive-ERM fails but \emph{regularized}-ERM succeeds. Thus, we view SCO as a natural test-bed for exploring the role of regularization and its relation to generalization. Compellingly, the generalization of SGD in SCO is well-established, and we are left with the question of how well can we account for generalization through an investigation of its bias.

\subsection{Contributions}

\paragraph{Implicit distribution-independent bias.}

We begin with a simple construction which demonstrates that SGD does not have any  \emph{distribution-independent} implicit bias. To show that, we construct a case where SGD \emph{does not} converge to a Pareto-efficient (not even approximately) solution with respect to the empirical loss and a given regularization penalty. 
In fact, this result is also true for Gradient Descent over smooth functions. In other words, our construction here involves a distribution supported on a single smooth convex function.

Our result is general and rules out any (reasonable) regularizer from being the implicit bias of SGD in this distribution-independent setting. Since the Euclidean-norm distance is the immediate suspect for the implicit regularization of SGD, the first step towards achieving the result is to rule out that Euclidean norm is the implicit bias of SGD. We thus construct an example of a function with a plateau of minimizers where SGD does not converge to the closest point in Euclidean-norm sense.
While the result might not seem surprising, it is the technical engine behind the further constructions we provide. Previous to this work, \citet{suggala2018connecting} showed that gradient descent with an infinitely small step size (that is, gradient flow), might diverge from the closest point, and we provide a complementary construction combined with a full rigorous analysis for fixed step-size gradient descent. 

\paragraph{Implicit distribution-dependent bias.}

Having ruled out the possibility of a problem-independent regularizer, we proceed to study the more compelling \emph{distribution-dependent} implicit regularization. The question here is whether for every distribution over convex functions, we can associate a regularizer $r$ such that SGD tries to (approximately) find a Pareto-efficient solution with respect to $r$ and the empirical loss (notice that we allow the regularizer to depend on the distribution, but \emph{not} on the specific sample received by SGD.)

We first show that we can rule out the effect of strongly-convex regularizers in the relevant regime of learning (where the dimension and the number of training examples are of roughly the same order). In fact, we rule out a more general class of regularizers that have large range on sets with large diameters. Namely, in any ball with large diameter the regularizer shows preference towards a certain point.

We then continue and demonstrate a distribution where, given an input sample, there is a very large set of possible solutions that share the same empirical loss and the same regularization penalty, and yet, SGD chooses its solution arbitrarily within this set. 
Here, by ``very large'' we mean from a learning-theoretical point of view; namely, this set is large enough so that, in general, 
empirical risk minimization restricted to the set will fail (and yet, it appears that this is exactly what SGD does).
In other words, no regularizer $r$ is sufficient for narrowing down the set of possible SGD solutions to the point where non-trivial generalization can be deduced without appealing to other properties of the specific problem.


\paragraph{Implicit bias in constant dimension.}
Several of our constructions are given in high dimension, namely the number of parameters is larger than the number of examples. One could argue that this is the interesting regime, nevertheless it is still worthy to understand the role of implicit bias when the dimension of the problem is smaller than number of examples.
Here we cannot rule out the role of implicit bias in a similar fashion to before - namely, due to uniform convergence, any algorithm that is constrained to the unit ball will generalize and this implicit bias is indeed the explanation to that. It is interesting though to understand the existence of specific regularizers (such as, e.g., strongly convex regularizers).

While we do not provide an answer to this question, we make an intermediate step. Our final construction is in a slightly relaxed model, where the instances are non-convex, but the expected loss function is convex. While this result may be limited, because of the non-convexity, we stress that the learning guarantees of SGD are completely applicable to this setting: namely, SGD does learn the problem (as it is convex in expectation). 
We show that for any \emph{strictly quasi-convex} regularizer, namely a regularizer that has preference for a single point in any convex regime, the algorithm will not converge to the optimal solution with optimal regularization penalty (even though it converges to a convex domain where seemingly it can improve its parameter choice towards the regularized solution).



\subsection{Related work}

Understanding the implicit bias of learning algorithms and its importance in generalization is a central theme in machine learning, and in the study of many classical algorithms~\citep{buhlmann2003boosting, schapire1998boosting, wei2017early}.
Recently, implicit bias has received considerable attention in the past few years.
Starting with \cite{NeyshaburTS14,ZhangBHRV17}, it was suggested that implicit regularization might explain the success of networks to improve test error by increasing network size beyond what is needed to achieve zero training error. 
Subsequently, a line of work has focused on identifying implicit regularization in various problems and domains, e.g., linear and non parametric regression~\citep{ali2019continuous, raskutti2014early, wei2017early} matrix factorization~\citep{gunasekar2017implicit, arora2019implicit}, linearly separable data~\citep{soudry2018implicit, gunasekar2018implicit}, as well as deep networks~\citep{neyshabur2017implicit,neyshabur2017geometry} and others~\citep{NacsonLGSSS19,nakajima2010implicit,lin2016generalization,gunasekar2018characterizing}. Our work here can be seen as an attempt to investigate the limitations of implicit regularization. 
Most similarly to this work, \citet{suggala2018connecting} provides an example of a problem where gradient flow does not converge to closest Euclidean solution. Here we focus on the more concrete SGD algorithm with a fixed step size, and give finite-time analysis. We are also able to harness our example to construct further new constructions that rule out a richer class of implicit-type regularization schemes.

This work can also be seen as an attempt towards separation between \emph{learnability} and \emph{regularization}. 
Besides regularization, several other useful notions have been suggested as surrogates of learnability. 
Most classically, uniform convergence \citep{blumer1989learnability} has been shown to be equivalent to learnability in the binary, distribution-independent model of PAC learning~\citep{valiant1984theory}. As discussed, \cite{shalev2009stochastic} showed that in the stochastic convex setting naive-ERM fails (but not regularized-ERM), hence learnability and uniform convergence are no longer equivalent. The constructions of \citet{shalev2009stochastic} were later substantially strengthened by \citet{feldman2016generalization}. 
More recently, \citet{nagarajan2019uniform} also provided an example that rules out uniform convergence, perhaps in the strictest sense. Their construction, though, does exhibit tangible implicit regularization, which account to the generalization of the algorithm.

Another useful notion is the \emph{stability} of a learning algorithm.
Stability is very much related to regularization: e.g., regularizing empirical risk minimization with a strongly convex function induces stability~\citep{bousquet2002stability}, and smoothness can also be harnessed to argue for stability \citep{hardt2015train}.
As such, constructing a convex problem where an algorithm is unstable could also serve as a means to rule out certain types of implicit regularizers. 
Our examples are in fact stable, and as such, could also be interpreted as a certain weak separation between stability and regularization.

\section{Preliminaries}

\subsection{The Setup: Stochastic Convex Optimization}
We consider the following standard setting of stochastic convex optimization. 
A learning problem consists of a fixed domain $\cW$, which for concreteness we will assume it to be a closed and bounded set in $\mathbb{R}^d$ for some finite $d$, a class of functions $f(\w;z)$ that are convex over $\w$, and an unknown distribution $D$ over a random variable $z$. 
The objective of the learner is to minimize:
\begin{align*}
F(\w) := \EE_{z\sim D}[f(\w;z)]
.
\end{align*}
The goal of the learner, given a sample $S= \{z_1,\ldots,z_T\}$ of $T$ i.i.d.~examples from the distribution $D$, is to
return a parameter vector $\w_{S}$ such that 
\begin{equation}
\EE_{S}[F(\w_S)]< \min_{\w\in \cW} F(\w)+\epsilon,
\end{equation}
for a desired target accuracy $\epsilon>0$. (The sample size $T$ may be determined based on $\epsilon$.) 

We make the following assumptions throughout.
We will generally assume that the functions $f$ are also $O(1)$-Lipschitz. Specifically, in all our constructions we will have $\|\nabla_\w f(w,z)\| \le 23$ for all values of $z$ and $\w \in \cW$. We will mostly be concerned with the case that $\cW$ is a bounded unit ball of radius $r$ around $0$. For concreteness we will mostly take $r=5$. This is just for convenience and clearly our results apply to any constant radius ball. Since our main focus in this paper is on impossibility results, fixing the Lipschitz constant and the diameter  does not harm the generality of the setup.

We will also discuss strongly-convex functions (or regularizers): we say that a convex function is $\lambda$\emph{-strongly convex} if for any $\w_1,\w_2 \in \cW$ we have: $f(\w_1)\ge f(\w_2) +\nabla f(\w_2)^\top(\w_1-\w_2) + \lambda \|\w_1-\w_2\|^2$.

\subsection{Gradient Descent and Stochastic Gradient Descent}

%
The main focus of this paper is the well-known Stochastic Gradient Descent (SGD) algorithm.
Given a sample $S=\{z_1,\ldots,z_T\}$ and a step-size parameter $\eta>0$, SGD initializes at $\textbf{w}^{(1)}=\textbf{0}$ and performs iterations:
\begin{equation}\label{SGD_alg}
    \forall ~ t=1,\ldots,T :
    \quad
    \w^{(t+1)} = \Pi_{W}\big(\w^{(t)}-\eta \nabla f(\w^{(t)}; z_t)\big)~,
    \quad
    \text{and outputs:}
    \quad
    \w_S
    = \frac{1}{T}\sum_{t=1}^T \textbf{w}^{(t)}
    ,
\end{equation}
where $\Pi_{W}(w)$ is defined to be the projection of $w$ over the convex set $W$.
%
%
The standard SGD analysis guarantees the following (see, e.g.,~\cite{shai_book}): 
\begin{theorem*}
Let $B, \rho>0$. Le $\cW=\{w: \|w\|\le B\}$, and
 assume that $F(\cdot)$ is convex and $\|\nabla f(w,z)\| \leq \rho$ for all $z$ and $w\in W$. 
Suppose that SGD is run for $T$ iterations on the sample $S=\{z_1,\ldots,z_T\}$ with step size $\eta=\sqrt{B^2/(\rho^2 T)}$. 
Then,
\begin{align} \label{thm:sgd}
    \EE_S[F(\w_S)]-F(\w^\star) 
    \leq 
    \frac{B \rho}{\sqrt{T}}
    ,
\end{align}
where here $\w^\star \in \arg \min_{\w : \|\w\| \leq B} F(\w)$.
\end{theorem*}

%
We will also discuss in this paper the procedure of \emph{Gradient Descent} (GD). Given an objective function $F$ GD obtains the following update steps:
\begin{equation} \label{GD_alg}
    \forall ~ t=1,\ldots,T :
    \quad
    \w^{(t+1)} = \Pi_{W}\big(\w^{(t)} - \eta \nabla F(\w^{(t)})\big)~,
    \quad
    \text{and outputs:}
    \quad
    \w_F
    = \frac{1}{T}\sum_{t=1}^T \textbf{w}^{(t)}
    .
\end{equation}
In our context, given a sample $S=\{z_1,\ldots, z_T\}$, the gradient descent algorithm takes steps using the full gradient with respect to the empirical loss defined as follows $F_{S}(\w)=\frac{1}{T}\sum_{t=1}^T f(\w,z_t)$. We will then write in shorthand $\w_S$ for $\w_{F_S}$
\ignore{
\begin{equation} \label{GD_alg}
    \forall ~ t=1,\ldots,T :
    \quad
    \w^{(t+1)} = \Pi_{W}\left(\w^{(t)} - \eta \nabla F_S(\w^{(t)})\right)~,
    \quad
    \text{and outputs:}
    \quad
    \w_S
    = \frac{1}{T}\sum_{t=1}^T \textbf{w}^{(t)}
    .
\end{equation}
where here $F_{S}(\w) = \frac{1}{T}\sum_{t=1}^T f(\w;z_t)$ is the \emph{empirical loss} of $\w$.}

\paragraph{Other variants of SGD.}

While the above version of SGD is perhaps the most standard one, there are other variants that can be considered.
For example, it is common to consider, instead of a fixed step-size, a decaying step-size (where $\eta$ may depend on $t$), as well as taking the last SGD iterate rather than the average iterate. 
We focus on the version in \cref{SGD_alg} for several reasons. First, taking the last iterate is not always justified and attains suboptimal rates (see \cite{shamir2013stochastic}). Second, the algorithm in \cref{SGD_alg} is also the more challenging variant to argue about, in the sense that averaging and taking small fixed step size induces bias towards initialization, and as such, is more strongly regularized (and indeed, the constructions we provide here can be readily modified to address a decaying step-size or the last iterate.\footnote{In fact, the proofs will be significantly simpler; for example, in the proof overview we actually consider the last iterate for simplicity.}) Another variant to consider is \emph{unprojected} gradient descent. Convergence bounds can be derived for this variant that depend on the norm of the benchmark solution \citep{shai_book, shalev2011online}. Again, we note that in all of our constructions we pick domain large enough so that projections in fact don't take place.

Nevertheless, it could be an interesting future work to derive a natural variant of SGD whose implicit regularization properties induce the desired generalization guarantees.

\subsection{Regularized (Structural) Risk Minimization}
Another well studied approach to perform learning is through \emph{regularization}, Regularized Empirical Risk Minimization (ERM) solves the following minimization problem:
\begin{align}\label{eq:rerm}
    \widehat{w}_\lambda 
    =
    \arg \min_{w \in \mathcal{W}} \big\{ F_S(\textbf{w})+\lambda r(w) \big\},
\end{align}
where $\lambda \in \mathbb{R}^+$, and $r(w): \mathbb{R} \mapsto \mathbb{R}^+$ is a regularization function. When $f(\w;z)$ is Lipschitz-bounded and $r(\w),\lambda$ are properly chosen this method leads to a principled learning algorithm. 
For example, in the case $r(\w)=\|\w\|^2$, \citet{bousquet2002stability} showed that with the correct choice of $\lambda$, Regularized ERM is guaranteed to generalize.

\section{Regularization}

We next discuss the different classes of regularizers we will consider in this paper. While some of the results we provide make little to no  assumptions on the regularizers, sometimes we would like to add further structure and rule out specific classes as the implicit bias of SGD, in other cases we would like to formally explain in what sense we might assume that the regularizer does not allow a comprehensive explanation of the implicit bias.

Most generally, a regularizer is any function $r:\cW\to \real_{+}$. We will however make the following basic assumptions on the regularizers, to avoid degenerate cases: 
\begin{itemize}
    \item $\min_{\w\in \cW} r(\w)=0$
    \item $r$ is non-constant at $\cW \backslash \{0\}$;
    \item $r$ is upper semi-continuous; namely, for every point $\w \in \cW$ and every $\epsilon_0>0$ there exists a neighborhood $B_{\delta_0}(\w)=\{\u: \|\w-\u\|<\delta_0\}$ for which $r(\u)> r(\w)- \epsilon_0$ if $\u\in B_{\delta_0}(\w)$.
\end{itemize}
Any regularizer that satisfies these properties will be said to be an \emph{admissible regularizer} (or shortly, a regularizer). The first assumption above is only for normalization. For the second assumption, the algorithms we will consider are all initialized at zero and may prefer the zero solution if it is a minimizer of the empirical error. But we are mostly concerned with the implicit bias in more involved cases then that.

The last assumption is perhaps somewhat strongest, but it is intended to rule out pathological examples. For example, one could consider a regularizer $r$ which is $0$ on almost all points, but is $1$ on the negligible, dense, set of real numbers that SGD would never reach. 
One could argue that $r$ is an implicit bias of SGD. However, this does not capture our intuition of a regularizer. Thus, we add an assumption that a point penalized by the regularizer should also be penalized under small perturbations. 

\subsection{Strongly-convex Regularizers}
While some of the results we will present are given for general (admissible) regularizers, it is natural and expected to study more structured classes of regularizers and ask if they induce the generalization properties of a certain algorithm. 
One natural family of such regularizers is the class of $\lambda$-\emph{strongly-convex} functions, which we will also assume are $1$-Lipschitz. 
As discussed in length, many of the prominent generalization results are provided in the context of strongly convex regularizers~\citep{bousquet2002stability,shalev2009stochastic}. 

Strongly-convex regularizers come with a very natural property which allows us to rule out such regularizers on certain problems: a strongly convex function always attains a \emph{unique} minimizer on any convex set. 
As such we can always identify if the output of an algorithm minimizes (approximately) the strongly-convex regularizer, by comparing the output to the minimizer of the regularizer over the given empirical risk.

\subsection{General (Admissible) Regularizers}
Studying implicit bias that does not stem from a strongly convex regularizer is no less important; however, it becomes much more subtle to rule out the latter.
Once the regularizer is allowed to have non-unique minima we should be more careful in stating what we mean when we say it \emph{does not explain generalization}. 
In fact, almost any plausible algorithm can be said to be implicitly biased on any given distribution.
For example, the fact that the regularizer is constrained to the unit ball is a form of algorithmic bias---but as was shown by \citet{shalev2009stochastic}, it cannot explain generalization in the SCO setting.

Towards clarifying what we mean by ``explain generalization'', let us consider the following: given a regularizer $r$ and an algorithm $\cA$ that outputs a solution $\cA(S)$ on a sample $S$, define the set of ``competitive'' solutions
\begin{equation} \label{eq:ksr} 
\begin{aligned}
    K_{S,r}(\cA(S))
    =
    \{\w \in \cW \;:\; F_S(\w) \le F_S(\cA(S)) \;\;\text{and}\;\; r(\w) \le r(\cA(S))\}
    .
\end{aligned}
\end{equation}
For shorthand, we will also use the notation $K_{S,r}(\cA)$ instead of $K_{S,r}(\cA(S))$. 

In words, $K_{S,r}(\cA)$ is the set of solutions that are comparable with (or better than) the output of~$\cA$, with respect to both the empirical loss and the regularization penalty.
For example, consider a regularized ERM, as in \cref{eq:rerm}, then $K_{S,r}(\cA)$ depicts \emph{all} minimizers of \cref{eq:rerm} with comparable regularization penalty. For example, with a strongly-convex regularizer $r$ one can observe that the set $K_{S,r}(\cA)$ is in fact a set of a single \emph{unique} solution. 

More generally, if a regularizer~$r$ is said to be the implicit bias of an algorithm $\cA$, and as such it explains the generalization of the algorithm, it is expected that the set $K_{S,r}(\cA)$ would be ``small'' in the sense that choosing an arbitrary solution from it should provide principled guarantees.
If we cannot attain such guarantees without further investigation of the problem and algorithm, we argue that the regularizer does not provide a comprehensive explanation of generalization. This motivates the following definition for studying more general regularizers than, say, strongly convex ones:

\begin{definition}\label{def:scomplexity}
Let us say that a set $K$ is $(T,\epsilon_0)$-\emph{statistically complex} if for some distribution $D$ over $1$-Lipschitz convex functions, given $T$ i.i.d.~samples we have that with probability at least $1/10$ that for some $\w\in K$ it holds that
$
    \frac{1}{T}\sum_{i=1}^T f(\w,z_i)=0,
    \;\text{yet}\;\;
    \EE_{z} [f(\w,z)]>\epsilon_0.
$
\end{definition}
Note that the statistical complexity of the set $K$ is measured with respect to an \emph{arbitrary} distribution $D$ over convex functions: this captures our requirement that the set $K_{S,r}(\cA)$ should explain generalization, \emph{without further investigation of the problem}. In other words, it could be that for a correct choice of a regularizer, on a specific problem, all the models in $K_{S,r}(\cA)$ will generalize. However, what we want is to ensure that the generalization does not stem from any further structure in the problem that is not captured by the regularizer. Thus, we require that this set will be ``simple'' in the sense that on any arbitrary distribution over convex functions we can choose an arbitrary solution that minimizes the empirical risk.
\section{Results}
\subsection{Distribution Independent Implicit Regularization}
We start with the natural question, whether there is some distribution independent implicit regularization being promoted by SGD. As a warm-up we begin by ruling out the existence of a distribution-independent strongly convex regularizer that plays the role of the implicit bias of SGD. This family of regularizers is already very interesting, and has been studied extensively in the literature of stochastic convex optimization~\citep{bousquet2002stability,shalev2009stochastic}.

\begin{theorem}\label{thm:gdwarmup}
Let $\cW=\{\w: \|\w\|\le 5\}$. For every $1$-Lipschitz and $\lambda$-strongly convex $r$, there is a distribution $D_r$ over 
 $1$-Lipschitz and $1$-smooth functions over $\cW$, 
and $\w_r \in \cW$ such that, with probability $1$,
SGD with any step size $1/T^2 < \eta < 1$ over an input sample $S$ of size $T = \Omega(1/(\lambda \eta))$ outputs $\w_S$ such that:
\begin{alignat*}{3}
    &F_S(\w_r) \le  F_S(\w_S),~\quad~
    \text{and}\quad
    &r(\w_r) \le r(\w_S) - \Theta(\lambda)~.
\end{alignat*}
\end{theorem}

In words, for any strongly convex regularizer there exists an instance problem where SGD chooses a solution that is sub-optimal in terms of both empirical error, and regularization penalty.

The last result can be extended to general (admissible) regularizers. Here, the rate of divergence from a Pareto optimal solution depends on the structure of the regularizer $r$. 
This dependence of the divergence-rate on the regularizer $r$ is unavoidable. Indeed if we consider a regularizer $r$ such that $r\approx 0$, it is not hard to be convinced that it would take SGD longer to become $r$-suboptimal.

\begin{theorem} \label{thm:gdr}
Let $\cW=\{\w: \|\w\|\le 5\}$. For every admissible regularizer $r$, there are constants $c_{r}>0$, a distribution $D_r$ (over $1$-Lipschitz and $1$-smooth convex functions), and $\w_r\in \cW$ such that, with probability $1$ over the input sample $S$, SGD with any step size $1/T^2<\eta<1$ and sample size $T_r= \Omega_r (1/\eta)$ outputs $\w_S$ such that: 
\begin{alignat*}{3}
    &F_S(\w_r) \le  F_S(\w_S),~
    \quad \text{and}\quad ~
     r(\w_r) \le r(\w_S)-c_r~.
\end{alignat*}
\end{theorem}

The $\Omega_r(\cdot)$ notation hides constant that may depend on the regularizer $r$. The dependence on the regularizer is expected here, as we would need a very strong level of accuracy if we want to rule out a nearly-constant regularizer, for example.

\subsection{Distribution-Dependent Implicit Regularization}

Having ruled out a class of implicit regularizers in the distribution-independent model, we next move on to discuss the possibility of distribution dependent regularizers.

\begin{theorem}\label{thm:sgdr} 
For every $T\ge 1$, a constant $C>2$ and dimension $d> T/10$:  there exists a distribution $D$ over $1$-Lipschitz convex functions over $\cW=\{\w: \|\w\|\le 1\}\subseteq \real^d$, such that if we run SGD with learning rate $1/T^2<\eta \le C/\sqrt{T}$ over a sample set of size $T$, then for any $1$-Lipschitz, $\lambda$-strongly convex regularizer $r$, with probability $0.1$ over the sample, SGD outputs $\w_{S}$ for which there is $\w^\star\in \cW$, such that 
\begin{align*}
    F_{S}(\w^\star) &\leq F_{S}(\w_S),~
    \quad
    \text{and}\quad
    r(\w^*) \leq r(\w_S) - 10^{-2}\frac{\lambda T\eta^2}{C}~.
\end{align*}
\end{theorem}
Utilizing a construction of a statistically complex set due to \citet{feldman2016generalization}, we can also obtain the following result:
 
\begin{theorem}\label{thm:nouc} For every $T\ge 1$, a constant $C>2$ and dimension $d\ge T/10^5$: there exists a distribution $D$ over convex $1$-Lipschitz functions over $\cW=\{\w:\|\w\|\le 1\}\subseteq \real^d$, such that if we run SGD with stepsize $1/T^2 <\eta \le  C/\sqrt{T}$ over a sample set of size $T$, then for any regularizer $r$ we have that with probability at least $1/10$ over the sample, the set $K_{S,r}(\w_{S})$ is $\left(2T,10^{-5}\frac{T\eta^2}{C}\right)$-statistically complex.
\end{theorem}

In words, \cref{thm:nouc} asserts that for a certain given distribution $D$ the output of SGD cannot be interpreted as coming from a ``small'' structured family of solutions that would generalize regardless of other specialized properties of the particular learning problem.

The requirement that $T\le O(d)$ is tight. Note that for a sample $S$ of order $T=O(d/\epsilon^2)$, by a standard covering argument, we can show that the set $K_{S,r}(\cW)$ is not $(2T,\epsilon)$ statistically complex (see, for example, Theorem 5 of \cite{shalev2009stochastic}). In particular, since $K_{S,r}(\w_S)\subseteq \K_{S,r}(\cW)$ we obtain an upper bound of the statistical complexity of the given set.

\subsection{Implicit Bias in Constant Dimension}

In the results above we provided constructions in spaces with more parameters than samples. We next discuss the case $d \ll T$, which is interesting for certain contexts.

Regarding \cref{thm:nouc}, we again point out that such a result cannot hold in the aforementioned regime. Indeed, in this case uniform convergence over the unit-ball applies. In that sense, restricting an algorithm to choose a solution in the unit ball provides an inductive bias that provides generalization guarantees. But what about \cref{thm:sgdr}? It is interesting to know if one can rule out regularizers that are not benign like the unit ball.\footnote{We treat a set $K$ as a regularizer by identifying $K$ with a regularizer $r$ such that $r(\w)=1$ if $\w\notin K$ and $0$ otherwise.}
We do not know the answer to this question and we leave it as an open problem. Nevertheless, we can provide the following intermediate result in a slightly more relaxed setting, where the instances may be non-convex, (and in fact non-Lipschitzian) but the expected loss function is indeed convex, and at each iteration the learner observes a bounded gradient $\|\nabla f(\w,z)\|\le 1$ Thus, SGD's learning guarantee still apply.

We will state the next result for a slightly larger class of regularizers than merely convex regularizers. Recall that a function $f$ is called \emph{quasi-convex} if $f(\lambda x+ (1-\lambda)y)\le \max \{f(x),f(y)\}$ for every $0\le \lambda\le 1$ and $x,y\in \cK$, and \emph{strictly} quasi-convex if $f(\lambda x+ (1-\lambda)y)< \max \{f(x),f(y)\}$.

\begin{theorem}\label{thm:nonconvex}
Let $\cW=\real^2$.
There exists a distribution $D$ over, not necessarily convex, functions in $\real^2$ such that $\EE_{z}[f(\w;z)]=0$ for every $\w\in \cW$,
and for every strictly quasi-convex regularizer $r$, and for large enough $T$, if $\eta=\Theta(1/\sqrt{T})$ then with some positive probability, $\Theta(1)$, there exists $\w^\star$ such that:
\begin{align*}
    \ERM(\w^\star) \leq \ERM(\ws);
    \qquad
    r(\w^\star) < r(\ws);
    \qquad
    \|\ws-\w^\star\| =\Theta(1).
\end{align*}
\end{theorem}

\section{Constructions}

\fullversion{Here we give a high level description of our first construction which rules out any norm-based distribution-independent regularizers and any strongly convex distribution-independent regularizer (\cref{thm:gdwarmup}). This construction is also the basis of the rest of the results (\cref{thm:gdr,thm:sgdr,thm:nouc,thm:nonconvex}).
The other constructions, as well as the full proofs, are provided in the full-version \cite{fullversion}.}{Here we give a high level description of the constructions as well as the proofs of the main results.}
We note that for simplicity of exposition, the following description refers to the last iterate, but our full proofs  refers to \cref{SGD_alg} (i.e., the algorithm that outputs $\w_S = \smash{\frac{1}{T}\sum_{t=1}^T \wt}$) .

\subsection{Distribution Independent Regularization}

Our constructions build upon the following class of functions in $\real^2$. Let $A$ be a set of the form $\{(\alpha,\theta): 0 \le \alpha \le b\}$, where $\theta, b$ are parameters of the set and $\Sigma$ is a PSD matrix. We then consider the function $f_{A,\Sigma}$ defined as follows: 
\begin{align}\label{eq:f_overview} 
    f_{A,\Sigma}(\w)
    =
    \tfrac12 \min_{\v\in A} \big\{ (\w-\v)^\top\Sigma (\w-\v) \big\}
    .
\end{align}
One can observe that these functions are convex, and further the gradient of $f_{A,\Sigma}$ at point $\w$ will equal
\begin{align}\label{eq:grad} 
    \nabla f_{A,\Sigma}(\w) 
    = 
    \Sigma(\w-\v(\w))
    ,
    \qquad\text{where}\qquad
    \v(\w)
    = 
    \argmin_{\v\in A} \{ (\w-\v)^\top \Sigma (\w-\v) \}
    .
\end{align} 

\paragraph{Warm-up: GD need not converge to a minimal-norm solution.}
We start by showing how we can construct a function (of the type in \cref{eq:f_overview}) that does not converge to minimal norm solution. Let us take a concrete case where
\begin{align*}
    A
    =
    \{(\alpha,1): 0\le \alpha \le \infty\}
    ,
    \qquad
    \Sigma = \begin{pmatrix}1 & \tfrac12 \\ \tfrac12 & 1\end{pmatrix}
    .
\end{align*}

We will suppress dependence on $A$ and $\Sigma$, and simply write~$f$.
The main observation is that the trajectory of $f$ is characterized by two phases.

At the first phase the closest point to $\smash{\wt}$ (with respect to the $\Sigma$-norm) is at the boundary of $A$ (i.e $\alpha=0$). At this phase, $\smash{\wt}$ can be seen to move ``towards'' the center of the interval, namely $\smash{w^{(t)}_1}$ is increasing (see \cref{eq:grad}). At the end of this phase, $\smash{w^{(t)}_1}$, is sufficiently large irrespective of the step size $\eta>0$.
The second phase, starts when $\etwo \equiv (\begin{smallmatrix} 0 \\1 \end{smallmatrix})$ stops being the closest point, and the closest point to $\smash{\wt}$ is at the interior of the interval. One can show that at this phase, the gradient moves upward hence $\smash{w^{(t)}_1}$ does not decrease and overall the trajectory will converge to a point away from $\etwo$: the Euclidean closest minimizer to~$0$. 

To see that when $\v(\w)$ is at the interior of $A$ then $\nabla f(\w) \propto \eone$, consider the following scalar function $g(a) = (\w- (a,1))^\top \Sigma (\w-(a,1)).$
Our assumption is that $g$ attains its minimum at $0<v_1$. Taking the derivative at $v_1$ and equating to $0$ (because the minimum is attained at the interior), we can see that $g'(v_1)= (\w-(v_1,1))^\top\Sigma \eone =0.$ Hence, $\nabla f(\w)= (\w-v(\w))\Sigma \perp \eone$. We depicted here the trajectory of GD without the projection step, however one can observe that throughout, the algorithm never escapes the $2$-ball, hence projections are indeed never implemented.
The trajectory of $\smash{\wt}$ is illustrated in \cref{fig:trajectory1} (green line). 

\paragraph{No strongly-convex implicit bias (\cref{thm:gdwarmup}).}

\begin{figure}[t]
\centering\small
    \begin{subfigure}{0.24\textwidth}
        \centering
        \includegraphics[width=0.95\linewidth]{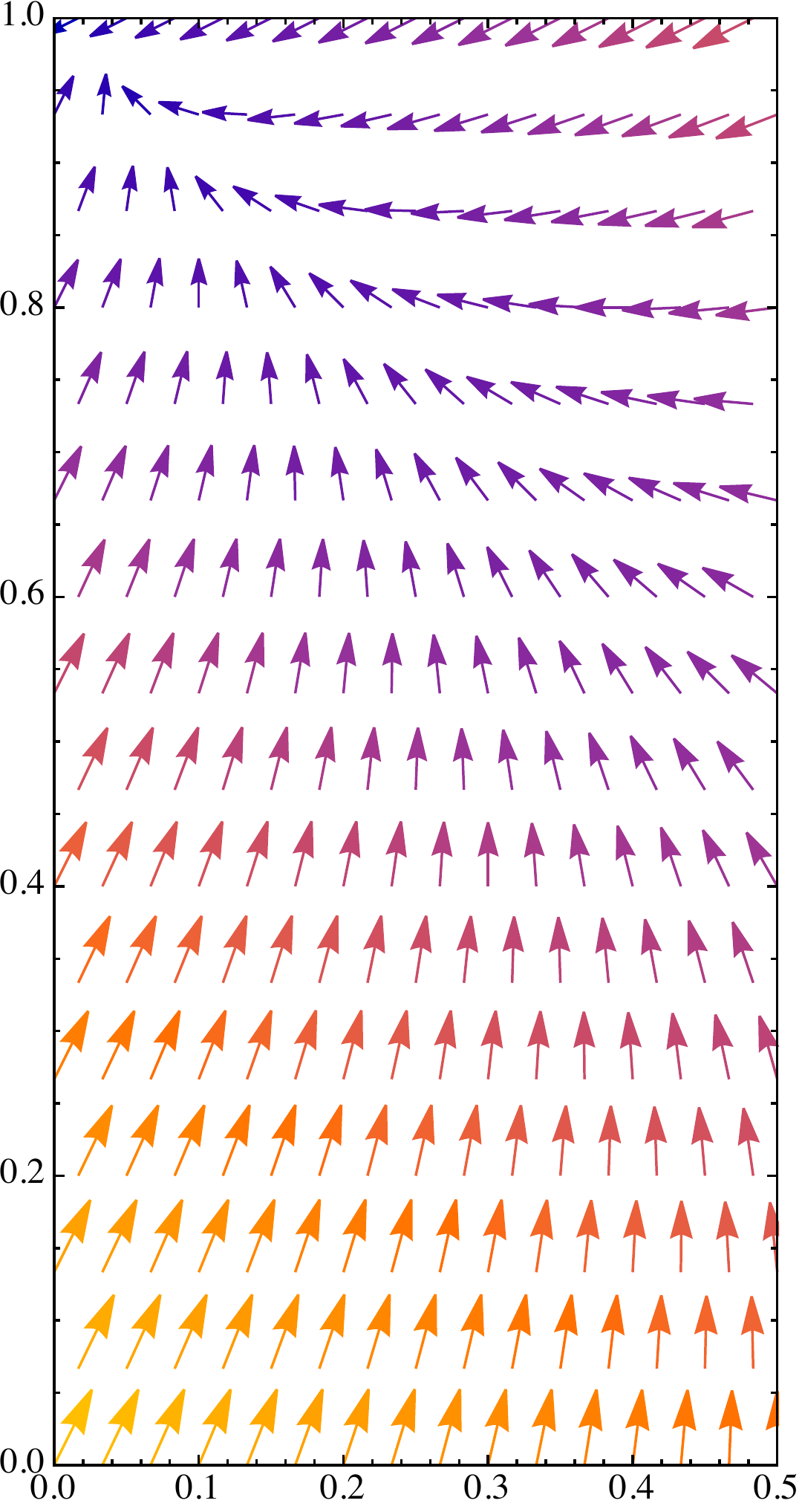}
        \caption{$b=0$;}
        \label{fig:field-a}
    \end{subfigure}
    \begin{subfigure}{0.24\textwidth}
        \centering
        \includegraphics[width=0.95\linewidth]{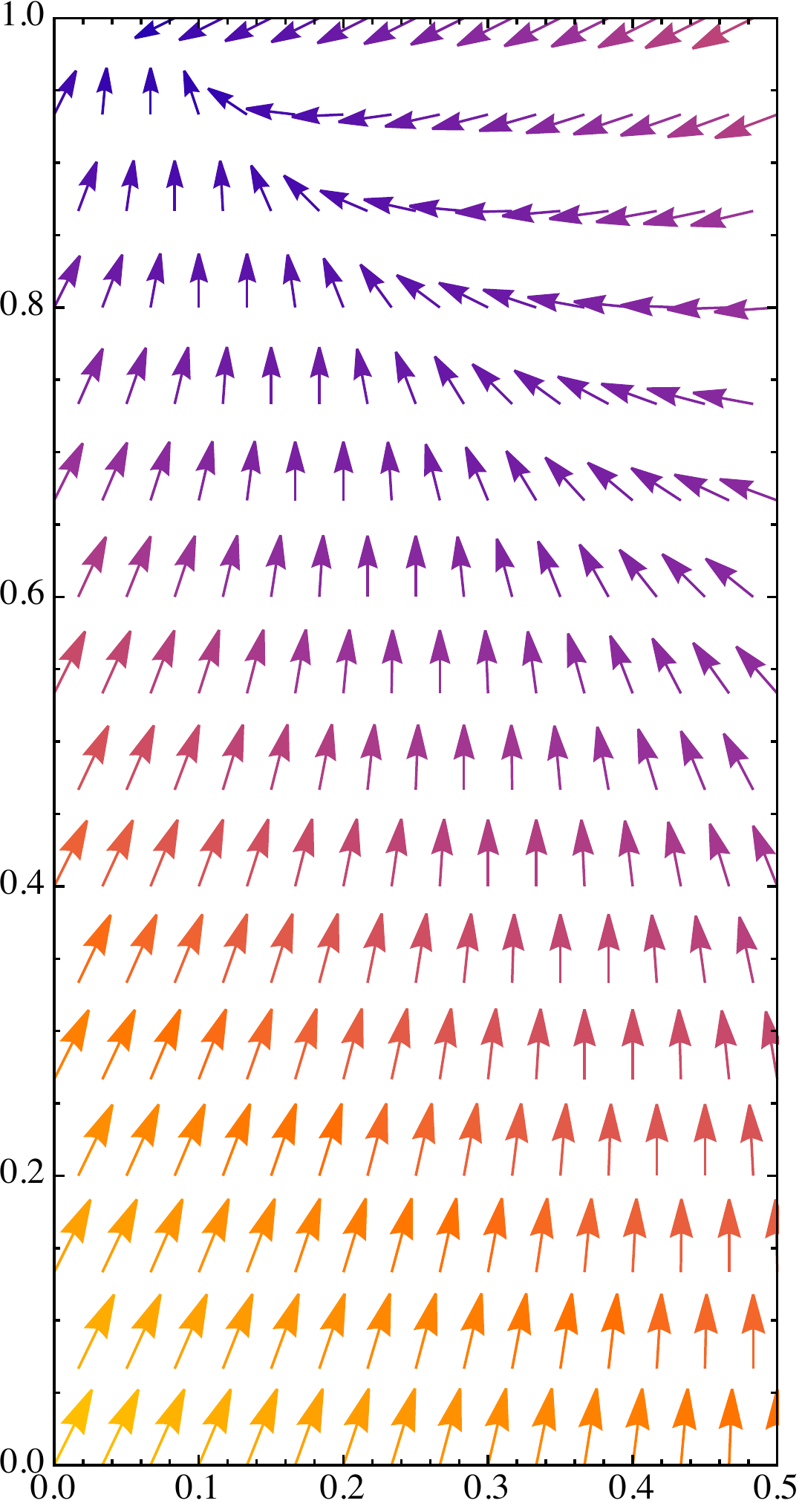}
        \caption{$b=0.05$;}
        \label{fig:field-b}
    \end{subfigure}
    \begin{subfigure}{0.24\textwidth}
        \centering
        \includegraphics[width=0.95\linewidth]{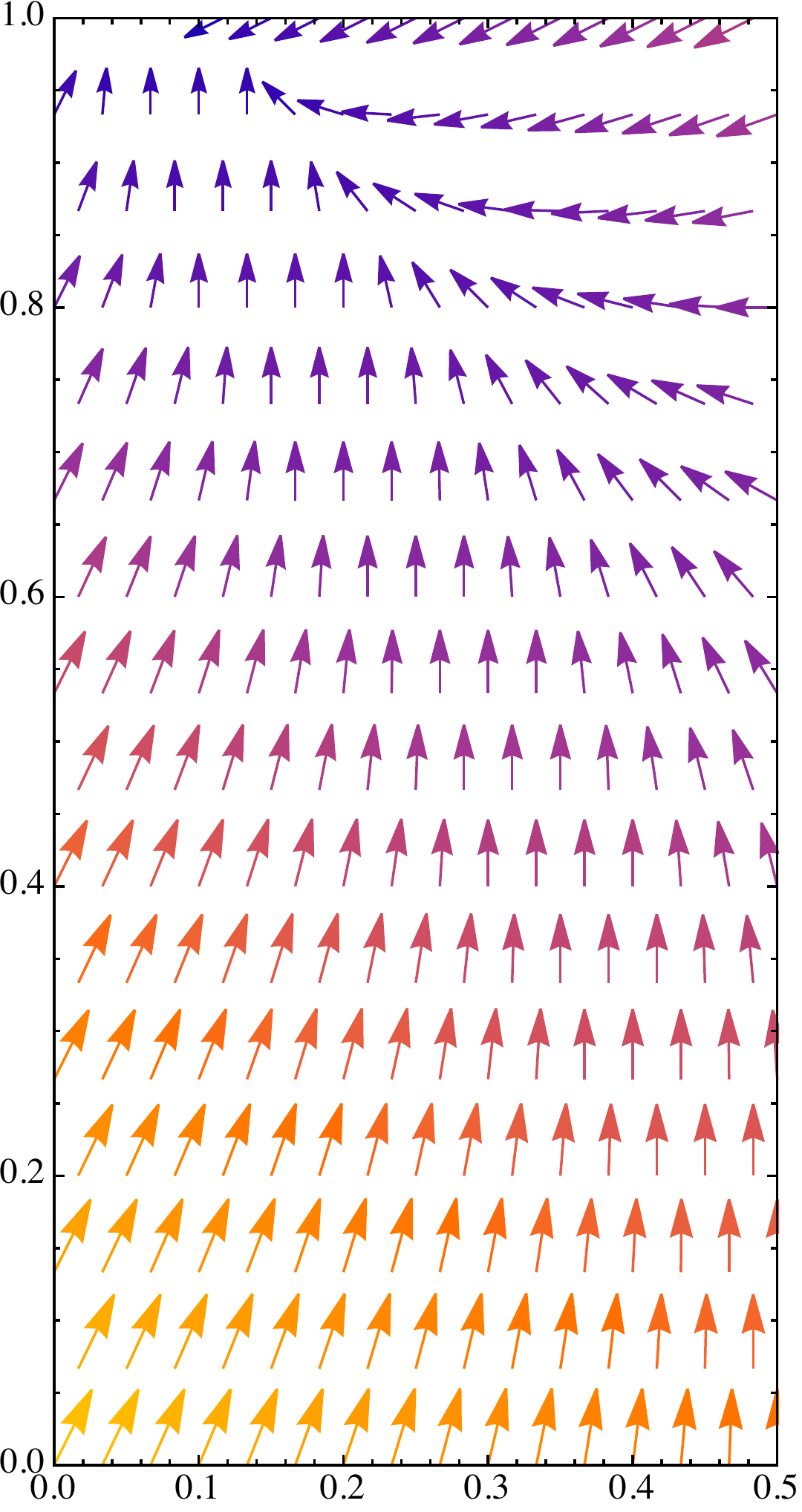}
        \caption{$b=0.1$;}
        \label{fig:field-c}
    \end{subfigure}
    \begin{subfigure}{0.24\textwidth}
        \centering
        \includegraphics[width=0.95\linewidth]{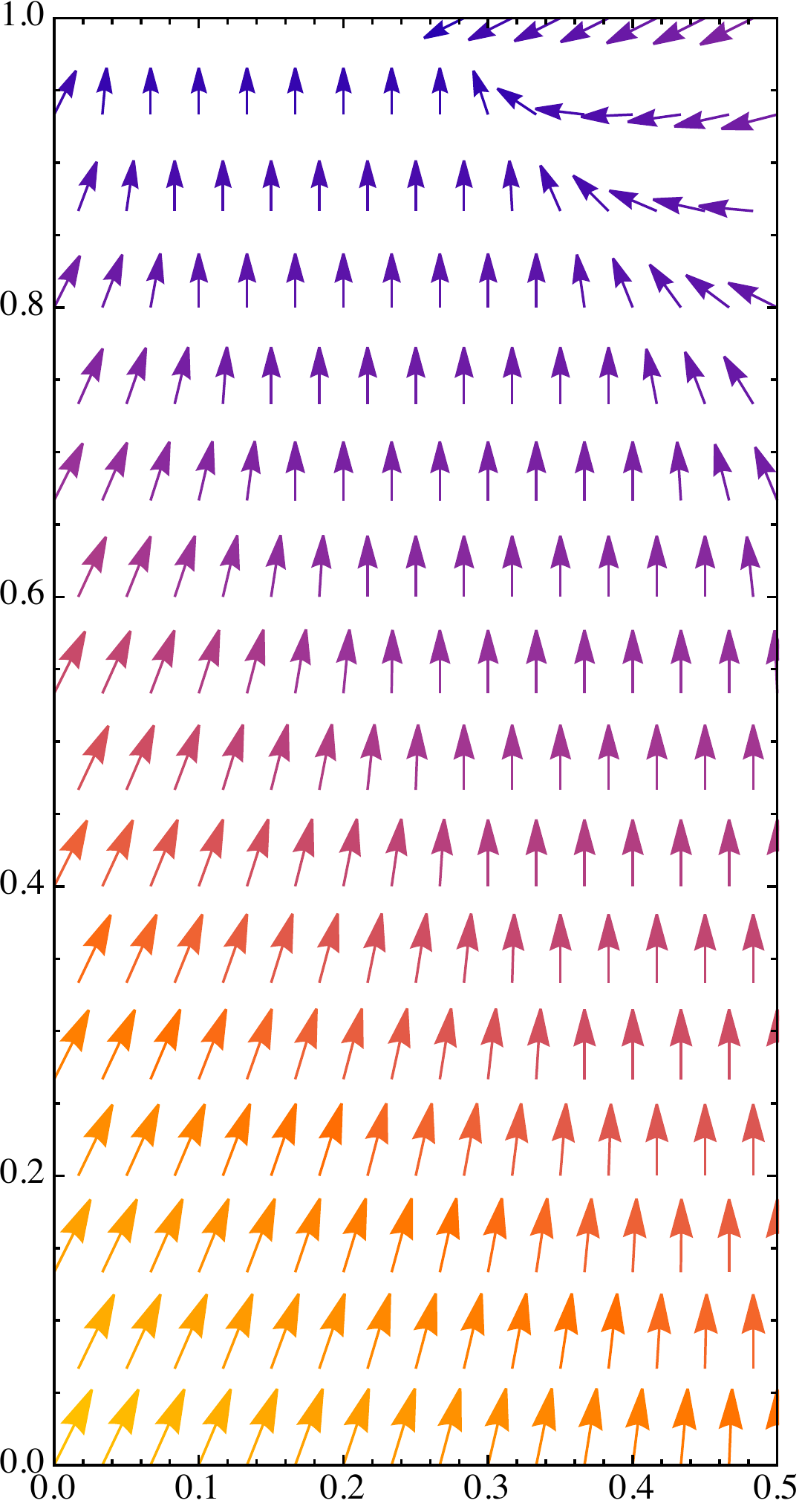}
        \caption{$b=0.25$.}
        \label{fig:field-d}
    \end{subfigure}
    \caption{\small%
    The gradient field of $f_{A,\Sigma}$ (see \cref{eq:f_overview}) for $\theta=1$ and varying values of $b$;
    near the origin, gradients (see \cref{eq:grad}) are skewed to the right, which causes GD to diverge from the nearest solution $\etwo=(0,1)$.
    }
    \label{fig:trajectory1}
\end{figure}

\begin{figure}[t]
    \centering\small
    \begin{subfigure}{0.24\textwidth}
        \centering
        \includegraphics[width=0.95\linewidth]{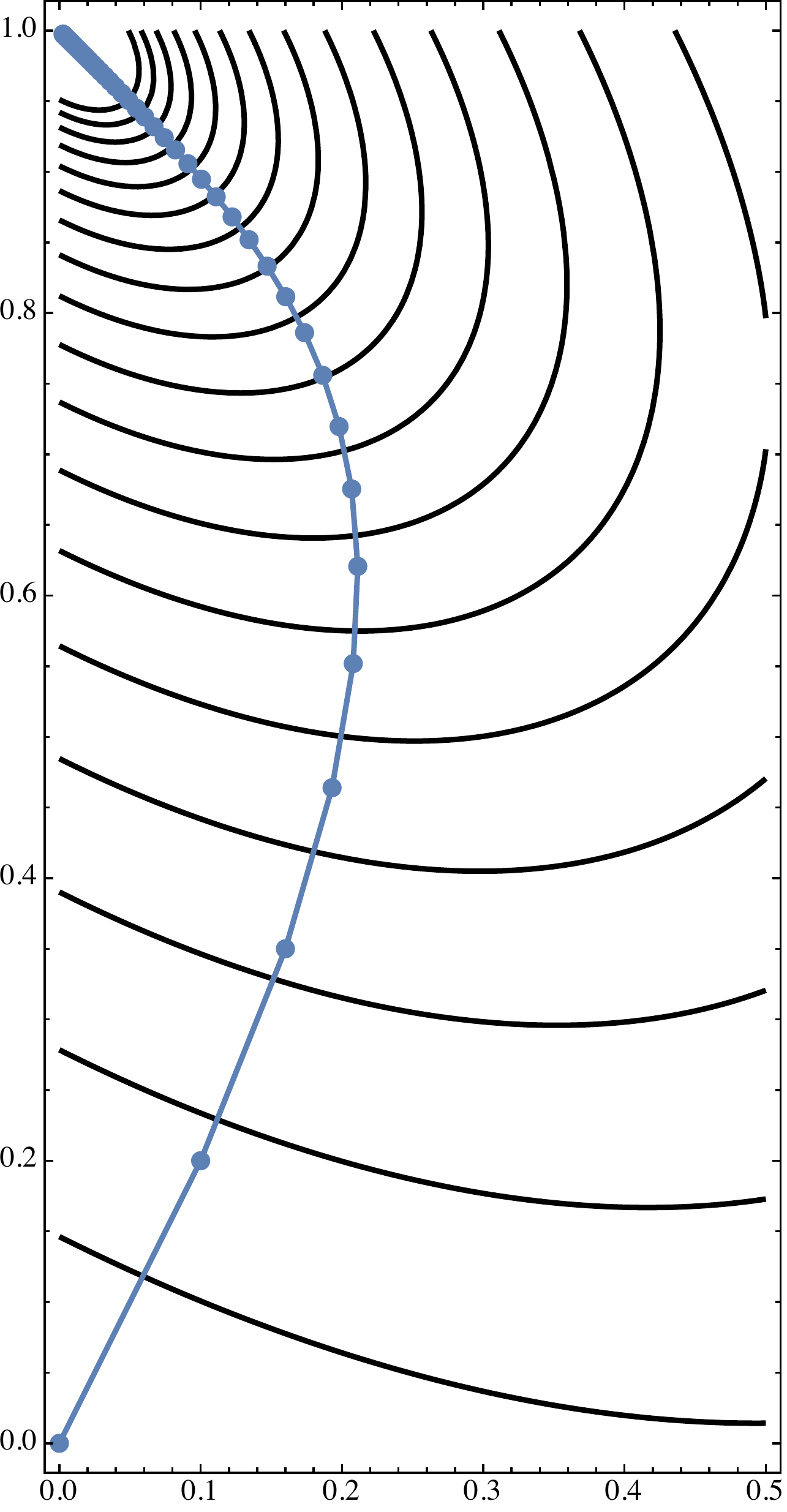}
        \caption{$b=0$;}
        \label{fig:trajectories-a}
    \end{subfigure}
    \begin{subfigure}{0.24\textwidth}
        \centering
        \includegraphics[width=0.95\linewidth]{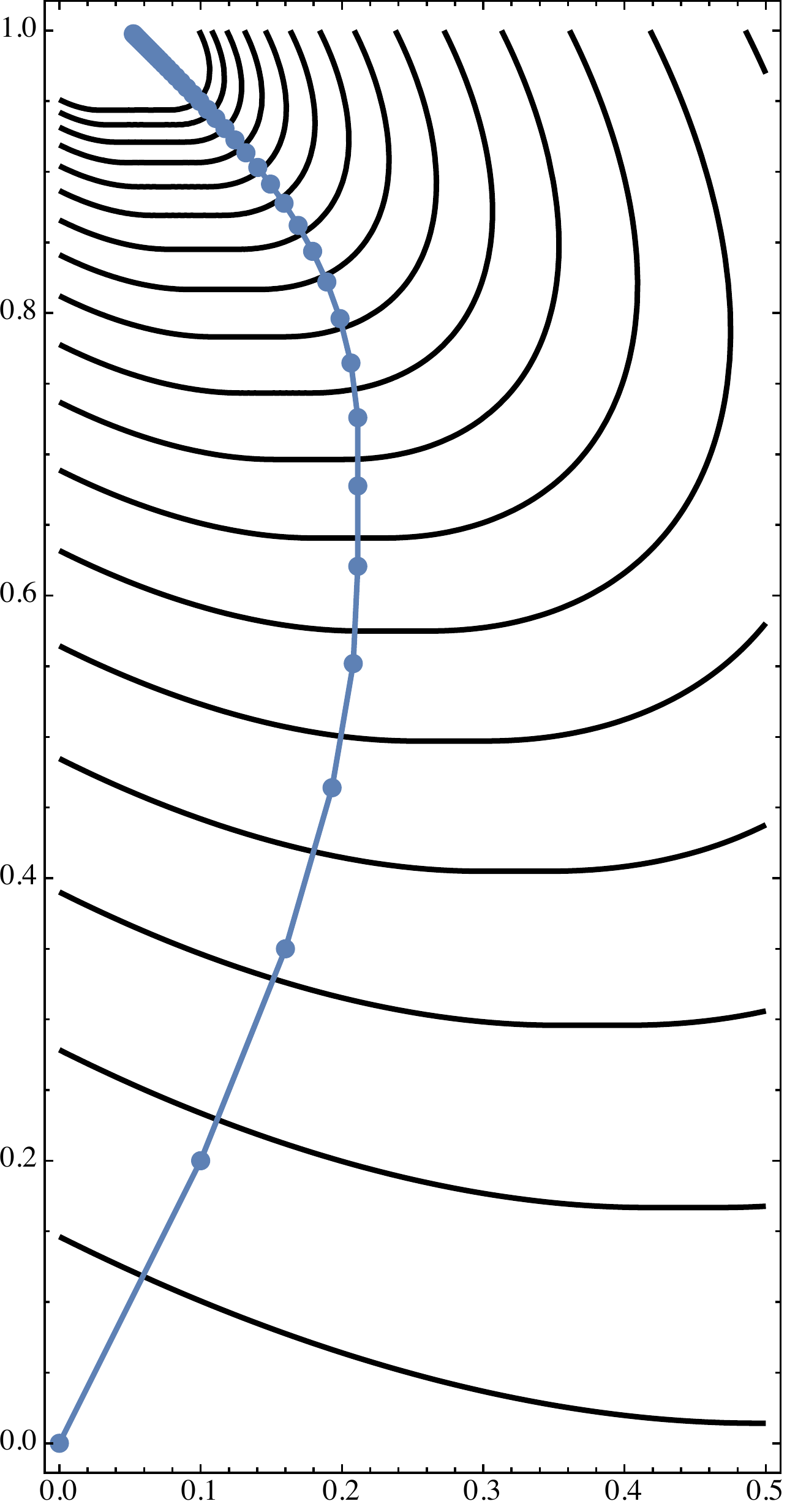}
        \caption{$b=0.05$;}
        \label{fig:trajectories-b}
    \end{subfigure}
    \begin{subfigure}{0.24\textwidth}
        \centering
        \includegraphics[width=0.95\linewidth]{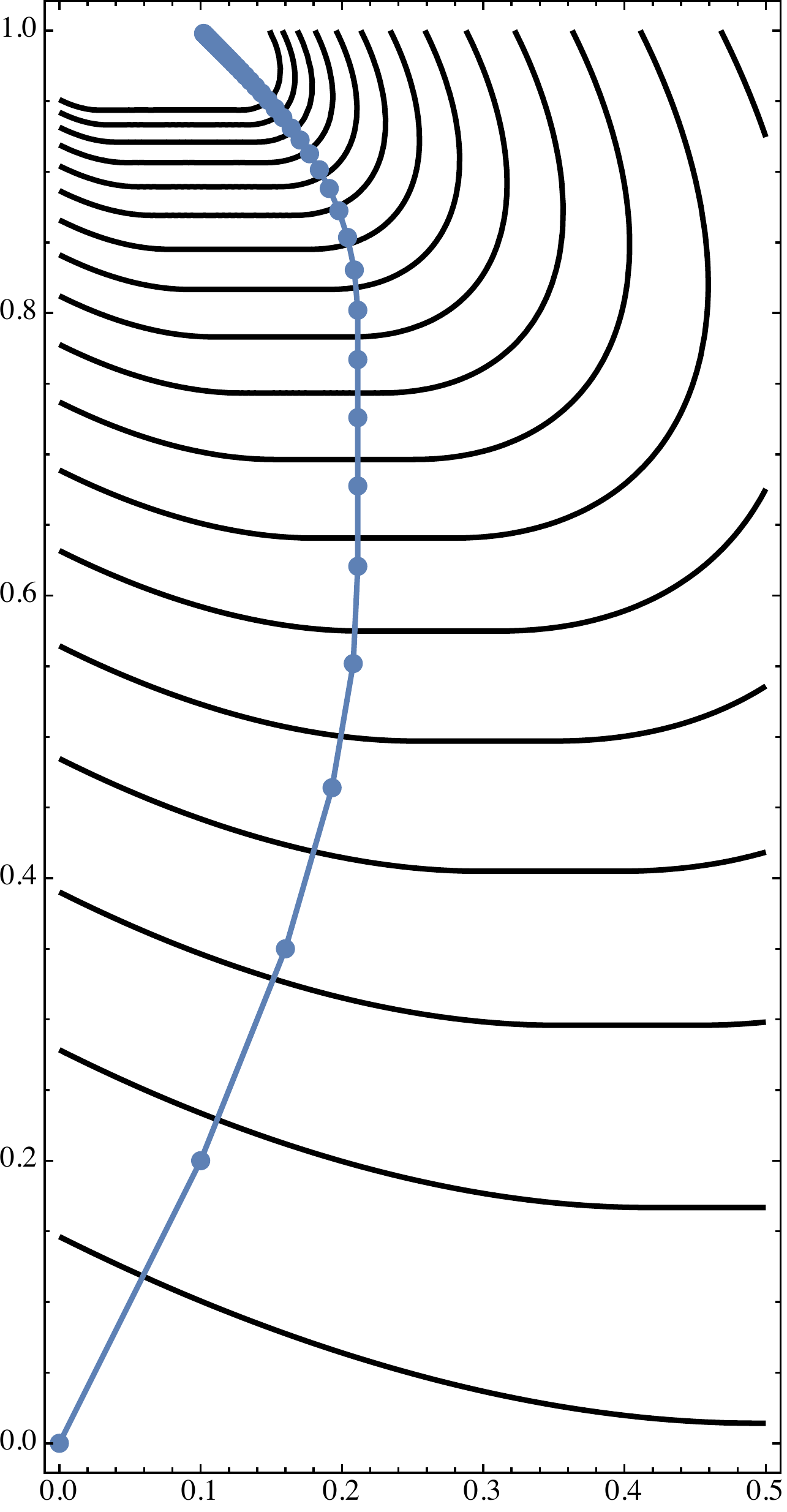}
        \caption{$b=0.1$;}
        \label{fig:trajectories-c}
    \end{subfigure}
    \begin{subfigure}{0.24\textwidth}
        \centering
        \includegraphics[width=0.95\linewidth]{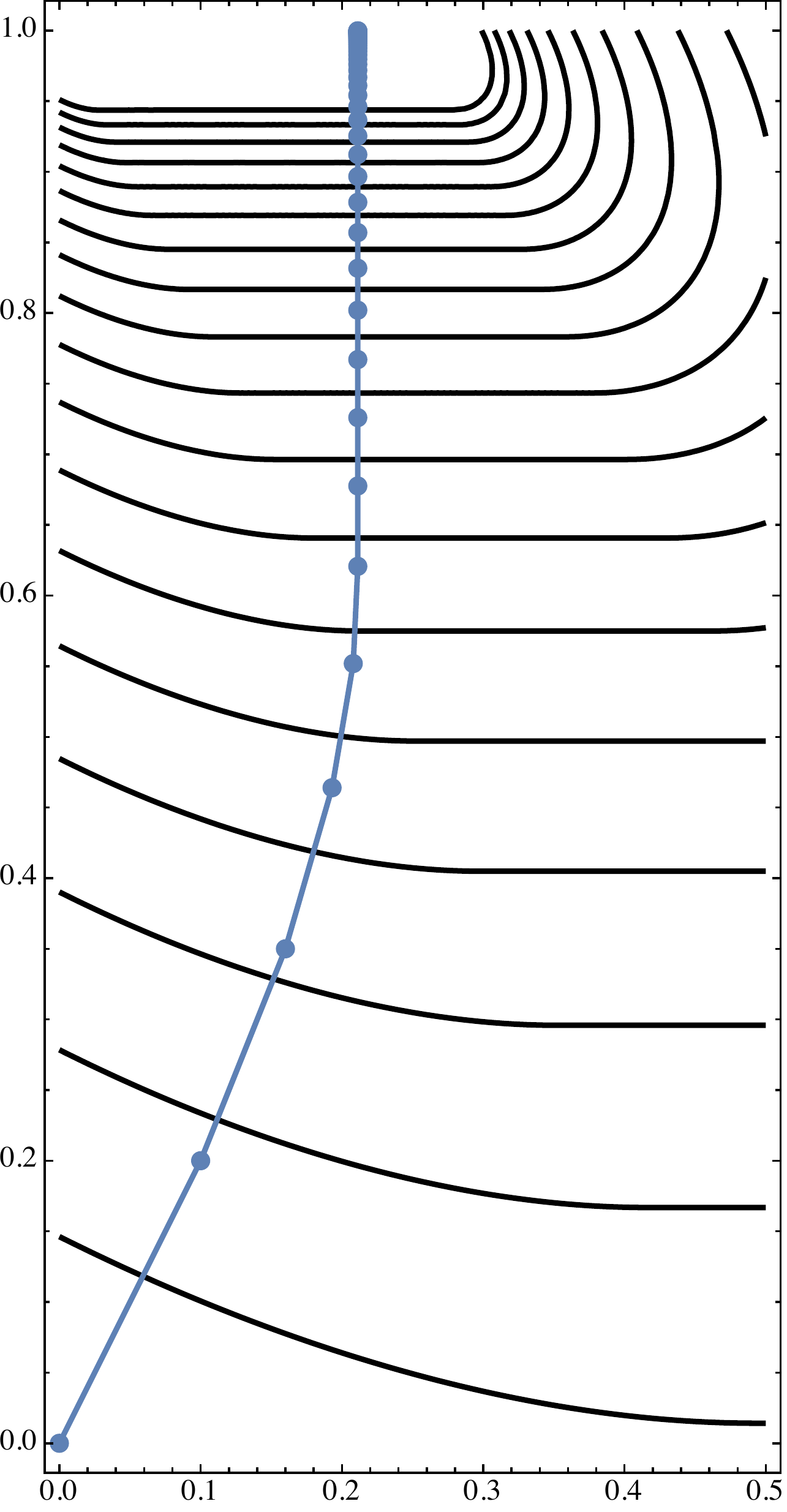}
        \caption{$b=0.25$.}
        \label{fig:trajectories-d}
    \end{subfigure}
    \caption{\small%
    Simulation of GD (with step size $\eta=0.2$) on $f_{A,\Sigma}$ for $\theta=1$ and varying values of $b$. We see that GD does not necessarily converge to the nearest solution, and tuning $b$ changes the point towards which it is biased.}
    \label{fig:trajectories}
\end{figure}

The construction above is the heart of most of our results.
Let us illustrate how it rules out a strongly convex regularizer (in the distribution-independent setting) and attain \cref{thm:gdwarmup}.

The key property of strongly-convex regularizers is that in any convex set they have a unique minimum. Moreover, two far away points cannot simultaneously attain close-to-minimal value. This is in fact the only property we will use. Thus, our result can in fact be extended to any regularizer that is a ``tie-breaker''---namely, it always prefers a single unique solution amongst a class of possible solutions with large diameter.

The construction above will allow us to generate two instances of convex learning problems, where SGD converges to two far away points. The first instance is the standard Euclidean distance. Namely, we take a function $f_{1}$ of the form in \cref{eq:f_overview}, with $\Sigma$ the identity and $A$ with boundaries $(-\infty,\infty)$. In this case, SGD is biased towards the nearest solution $\etwo=(0,1)$. 
The second instance, $f_2$, is the construction above where SGD is biased towards another point on the interval (see  \cref{fig:trajectories-b,fig:trajectories-c,fig:trajectories-d}).

Now both points are global minima, for both $f_1$ and $f_2$, hence if SGD is implicitly biased towards solutions with minimum regularization penalty $r$, we must have that $r(\etwo)= r(\v)$, where $\v$ is the choice of SGD when it observes $f_2$.
However, if $r$ is strongly convex,
because $\|\etwo-\v\|=\Theta(1)$, there has to be a point on the interval between them that attain a strictly lesser regularization penalty, moreover it also attains minimal loss value. This contradicts the existence of such an $r$.

\paragraph{The general case.}

Our second result (\cref{thm:gdr}) rules out the existence of any distribution-independent regularizer. In contrast with the strongly-convex case we can not give uniform bounds that depend on parameters of strong convexity. As such, the rates depend on the regularizer.

But the construction here is similar. We basically start with the assumption that there are two points $\w_1$ and $\w_2$ with different regularization penalty, and we want to construct two functions $f_1,f_2$ that maps $\w_1,\w_2$ to the same empirical loss. 
It might seem that through a simple linear transformation that maps, say, $\w_1$ to $\etwo$ and $\w_2$ to $\v$ we can reduce this case to the case above. However, there is some subtlety since gradient descent is not invariant to linear transformations.%
\footnote{We note though that it can be turned to an affine invariant optimization algorithm~\citep{koren2017affine}.}

Towards this, we extend the construction above by constructing a more general example, where we can tune the point of convergence of SGD to \emph{any} point on the interval between $\v$ and $\etwo$. This allows us to avoid scaling, and use only rotations (which SGD is invariant to) in order to reduce the problem to the former case.
This is done by changing the set $A$ from allowing $0\le \alpha < \infty$ and $\theta=1$, to adding a second boundary condition on the right and also scaling $\theta$. In \cref{fig:trajectories} we illustrate how changing the boundary condition changes the trajectory.

\subsection{Distribution-Dependent Implicit Bias}

We next discuss our second sets of results that argue about \emph{distribution-dependent regularization}. Here we want to study if, for a given distribution, the set of solutions on which SGD converges has some meaningful structure on which we can argue why it generalizes.

Note that so far, our problem instances considered only a single function and the results were applicable to GD also. Here, though, in the distribution dependent setting such an example cannot work. Indeed, given a single function as an instance problem, SGD behaves deterministically and the solution it chooses is a unique solution which trivially generalizes.

\paragraph{No strongly-convex distribution-dependent bias.}

We next discuss our argument that rules out a strongly convex regularizer, even if it may depend on the distribution at hand. We again utilize the property that a strongly convex regularizer obtains approximately minimum solutions only on a small diameter around the unique minimum.

Our strategy is as follows: assume that there are two samples $S_1$ and $S_2$ such that, when SGD observes $S_1$ it converges to $\w_1$ and when it observes $S_2$ it converges to $\w_2$. However, assume also that $\|\w_1-\w_2\|=\Theta(1)$, and that the empirical loss of $\w_1$ and $\w_2$ is comparable, on both samples: namely $F_{S_1}(\w_1)=F_{S_2}(\w_1)$, and similarly with $\w_2.$

In the case above, as we argued in the distribution-independent case, clearly the algorithm failed to choose the minimizer of the regularization penalty, in at least one of the realization $S_1$ or $S_2$. So if $S_1$ and $S_2$ are equally likely, we obtain that with probability half (conditioned on the event that we saw $S_1$ or $S_2$) the algorithm failed to minimize $r$.
Now, if the probability to observe one of such couple of samples $S_1,S_2$ is positive, then we obtain the desired result.

To generate this setting, we rely on the following auxiliary construction in $\real^2$. We construct two functions such that, if SGD observes the first function, at the first iteration, then the gradient points upward and right. But if SGD observes the second function, at the first iteration, then the gradient points upward. This ensures that in each case SGD will move towards a different solution. If the size of the gradient is constant then the gap between the two iterations will be $\Theta(\eta)$.

We will also construct the examples in such a way that both points enter a regime where all points obtain the same empirical loss on both functions. This construction can in fact be done using piece-wise linear functions and it is illustrated in \cref{fig:high_dim}. We also give the formal statement here:
\begin{restatable}{lemma}{rtwo}\label{lem:r2}
For every constant $0<c<1$, there are two $1$-Lipschitz functions $f(\w; \pm1)$ over $\real^2$ such that if $\v_1=-\nabla f(0;1)$ and $\v_{-1}=-\nabla f(0;-1)$ and $ c <\frac{1}{2}\eta<1$ then $\|\v_1-\v_{-1}\|\ge 1/4$ and $f(\eta \v_1;z)=f(\eta \v_{-1};z)$ for any $z\in \{-1,1\}$.
\end{restatable}

\begin{figure*}[t]
    \centering
    \includegraphics[width=0.7\textwidth]{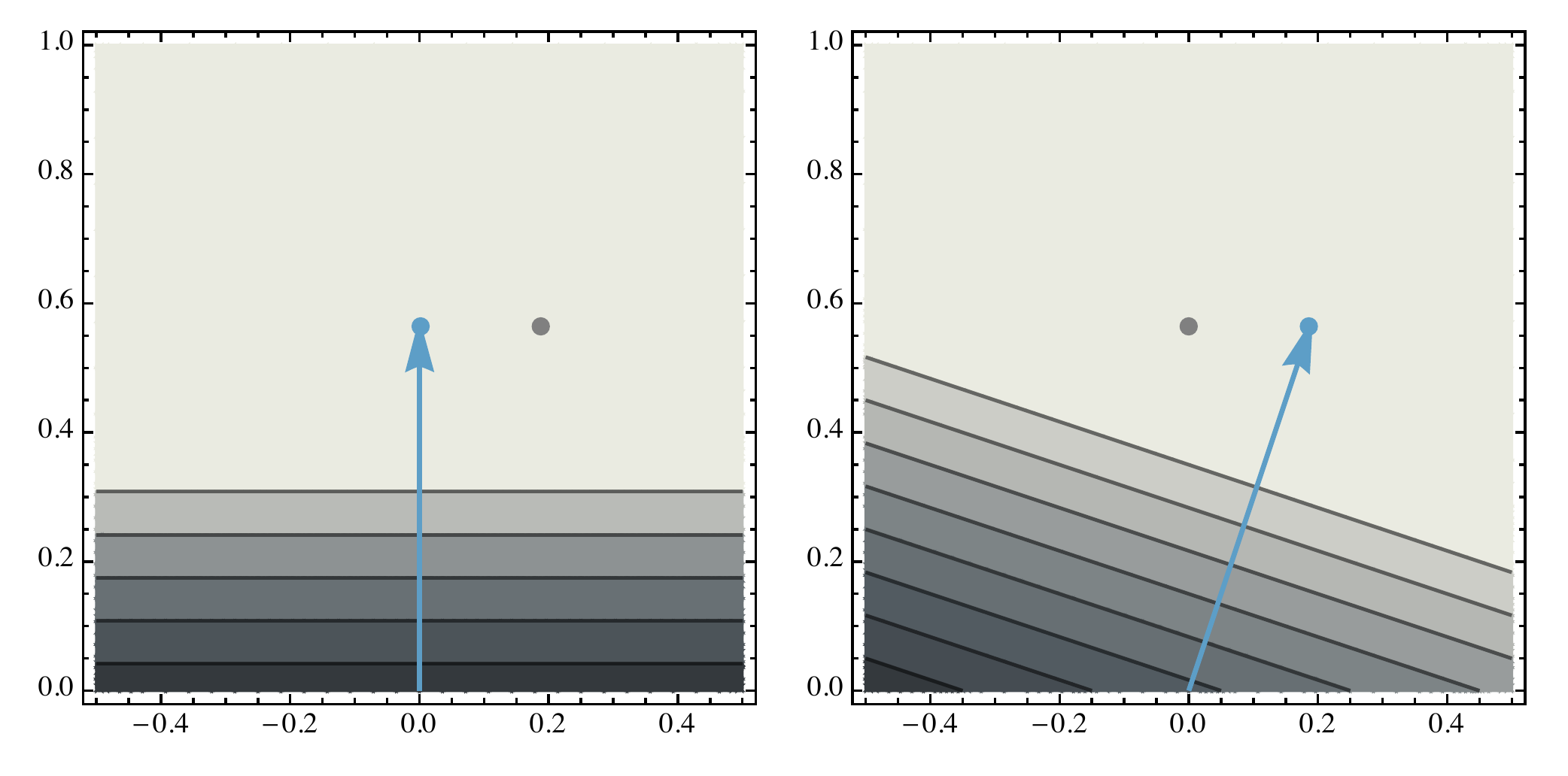}
    \caption{\small%
    Depiction of the auxiliary construction in \cref{lem:r2}. The left sketch illustrate the first function, and the right sketch the second function. Both functions are piece-wise linear of the form $x\to \max\{0,\v\cdot x\}$. The left function, the gradient of the loss points upward at the origin and is flat at a second regime. The right function, the gradient at origin points sideways, and again flat at the second regime.
    }
    \label{fig:high_dim}
\end{figure*}

We next utilize the above construction to generate the problem in $\real^d$. Note that the construction above generates a problem where SGD will converge to two different solutions with distance $\eta$ but same empirical loss (after one step). Indeed, we just need to randomly pick one of these functions. 

We next want to amplify the distance. To do that, we consider $d=\Omega(T)$ Cartesian copies of $\real^2$. Then at each example, we show one of the functions above, at one of the products.
Assuming enough coordinates were seen only once (which is going to happen w.h.p.), the variance on each sub-plane will be $\eta^2$: if we have $\Theta(T)$ such coordinates, the overall variance is going to be $\Theta(T\eta^2)$ which ensures that we will converge to far away solutions on different realizations of the problem, if $\eta=\Theta(1/\sqrt{T})$.

\paragraph{SGD might be biased towards statistically-complex sets.}

Next, we derive \cref{thm:nouc} which addresses implicit regularization in a much broader setting. As discussed, here we cannot rule out the existence of an implicit bias; indeed, some form of an implicit bias always exists. 
We attempt, though, to understand how the implicit bias can explain generalization.

The result shows that for any regularizer: the set $K_{S,r}(\w_S)$ which is the set of comparable solutions to the one outputted by SGD, given the empirical loss and regularization penalty, can be large up to the fact that choosing an arbitrary solution from this set can, in principle, lead to over-fitting (over general convex problems). Thus, to argue that the algorithm did generalize, further structure in the problem needs to be taken into account. And this is true for \emph{any} regularizer.

Our construction is similar to the previous case in \cref{thm:sgdr} up to some modification. Therefore, let us show that in the construction above $K_{S,r}$ will be $(T/6,\Theta(1))$-statistically complex. 
This is less than what we actually desire. We, in fact, observed $T$ examples and not $T/6$. Indeed, in the construction above, we showed that if we project the output of  SGD to the observed coordinates, we obtain a solution of the form $(\v_{\pm 1},\v_{\pm 1},\cdots, \v_{\pm 1})\in (\real^2)^{T}$, where $\v_1,\v_{-1}$ are as in \cref{lem:r2}.
By projecting this set, it can be seen to be a copy of (up to some rescaling) the normalized unit cube $\cM=\{\pm\eta, \pm\eta ,\cdots, \pm\eta\} \in \real^T$. This is true since $\|\eta \v_{1} - \eta \v_{-1}\| = \Theta(\eta)$.

Here, we rely on a construction by \citet{feldman2016generalization}. 
In order to show that uniform convergence is not equivalent to learnability in the convex optimization setting, Feldman showed (in our terminology) that the set $\cM\in \real^T$ is $(T/6,1/4)$ statistically complex, if $\eta=\Theta(1/\sqrt{T})$.

As discussed, this is less than what we want, as we actually want a set that is at least $(T,\Theta(1))$ statistically complex. To tackle this, on each iteration we show the learner a loss function over multiple pairs of coordinates. 
Namely, if in the example above we drew at each iteration $f(\w;z)$ where $z\sim D$, now in each iteration we show the algorithm $\frac{1}{k}\sum_{i=1}^k f(\w;z_i)$, where $z_i$ are i.i.d. 
This will reduce the step-size on each coordinate a little bit but if $k$ is constant we will still present a constant loss.
On the other hand, now projecting on observed coordinates, SGD will converge to a solution in $(\v_{\pm 1},\v_{\pm 1},\ldots,\v_{\pm 1})\in (\real^2)^{\Theta(kT)}.$ Thus we only need a constant $k>6$ so that the algorithm will converge to a $(T,\Theta(1))$-statistically complex set.

\subsection{Implicit Bias in Constant Dimension}
We next provide a construction in $\real^2$ that again rules out a class of regularizers, in particular strongly convex regularizers (and more generally, strictly quasi-convex regularizers).

In a similar fashion to previous constructions, we make SGD choose from a set of solutions, that exhibit comparable empirical loss. While the dimension of previous constructions depended on $T$, this construction does not. However, for the construction we relax the assumption that $f(\w ;z)$ are convex, but $F$ remains convex. Note that the learning guarantees of SGD are completely applicable to this setting.

Our construction relies on a 2-dimensional square, centered at the origin. Inside the square, SGD makes a simple 2-dimensional random walk, while when it exits from the square, it continues to perform a random walk in just one dimension (denoted as  $y$), while the other coordinate (denoted as  $x$) remains the same. As a result, the optimizer of $\ERM$ is independent of $w_x$.

We study the event that $\w$ will stay inside the square for enough iterations to ensure that the variance of $w_x$ will be larger than some constant, but eventually $\w$ exit from the square to make $\ERM$ independent of $w_x$. This will result with a set of solutions that share the same empirical error and also SGD can converge to each one of them.

\bibliography{sample}

\onecolumn
\appendix

\section{Technical Background}

\subsection{Feldman's Statistically Complex Set}
A key technical tool in the proof of \cref{thm:nouc} is a construction by Feldman, \cite{feldman2016generalization}, of a statistically complex set in $\real^d$. While Feldman's construction is not the first to show that the sample complexity of an ERM algorithm may scale with the dimension, it greatly improved over previous construction \cite{shalev2009stochastic}, and showed that the dependence may be \emph{linear} in the dimension. 

We will exploit here Feldman's set in order to construct an example where SGD essentially picks arbitrarily an element from a statistically complex set, akin to ERM, and we will need the following statement due to Feldman
\begin{theorem}[Essentially Theorem 3.3 in~\citealp{feldman2016generalization}] \label{thm:feldman}
Let $\cW_d=\{-\frac{1}{\sqrt{d}},\frac{1}{\sqrt{d}}\}^d.$ There exists a distribution $D$ over $1$-Lipschitz convex functions such that given a sample $|S|<d/6$ drawn i.i.d from $D$ then w.p. $1/2$ (over the sample $S$) there exists $\w\in \cW_d$ such that
\begin{align}
\frac{1}{|S|}\sum_{t=1}^{|S|} f(\w,z_t) =0,
\end{align}
but
\begin{align}\EE_{z\sim D}[f(\w,z)] =1/4.\end{align}
\end{theorem}
We will need a slightly stronger version of the theorem which is an immediate corollary
\begin{corollary}\label{cor:feldman}
Let $A\subseteq \cW_{d}$, such that $|A|\ge 2^{d-1}$, then $A$ is $(d/6,1/4)$-statistically complex.
\end{corollary}
\begin{proof}
For two vectors $\v\in \{-1,1\}^d$ and an element $\w\in \cW_d$ let $\v*\w\in \cW_d$ be the pointwise product between $\w$ and $\v$, i.e.
\[(\v*\w)_i = \v_i\cdot \w_i.\]

Let $D$ be the distribution from \cref{thm:feldman} and consider a distribution where we draw uniformly an elements $\v\in \{-1,1\}^d$ and a sample $S$ of size d/6 i.i.d from $D$. One can show that with probability $|A|/(2^{d+1})$ we have that there exists an elements $\w\in A$ such that
\begin{align}\label{eq:feldmanerm}
\frac{1}{|S|}\sum_{t=1}^{|S|} f(\v*\w;z_t) =0
\end{align}
but
\begin{align}\label{eq:feldmanrisk}\EE_{z\sim D}f(\v*\w;z) =1/4.\end{align}

In particular, there exists a $\v$ such that with probability $\frac{|A|}{2^{d+1}}$, \cref{eq:feldmanerm,eq:feldmanrisk} holds for some $\w\in A$ over the random sample $S$. Thus, we can define a convex Lipschitz mapping parameterized by $\z$ such that
\[f_{\v}(\w;z)=f(\v*\w;z).\]
From the above discussion if we draw $z\sim D$ we can see that this distribution demonstrates that $A$ is (d/6,1/4)-statistically complex

 \end{proof}
 
 \subsection{Berry-Esseen Theorem}
 A very important and valuable tool for analysing the behavior of random walks that we will use is the well-known Berry-Esseen Theorem, discovered independently in \cite{berry1941accuracy,esseen1942liapunov}.
 
 \begin{theorem}[Berry Esseen Theorem]\label{thm:be}
 Let $X_1,X_2,\ldots, X_T$ be zero mean and independent random variables, with $\EE(X^2_i)=\sigma_i^2$ and $\EE(|X^3_i|)=\rho_i$. Let $S_T=\frac{1}{ \sqrt{\sum_{i=1}^T\sigma_i^2}}\sum_{i=1}^T X_i $, then we have 
 \[|P(S_T\leq a)-\Phi(a)|\le \cbe (\sum_{i=1}^T \sigma^2_i)^{-3/2}\sum_{i=1}^T \rho_i,\]
 where $\cbe<1$ is an absolute constant, and $\Phi(a)$ is the CDF of a unit variate zero-mean Gaussian random variable.
 \end{theorem}
 
 For a bound $\cbe<1$ of the absolute constant see, for example, \citet{van1972application}. We will need the following technical Lemma which is derived via \cref{thm:be}:
 \begin{lemma}\label{cor:be} Let $k\ge 0$, and assume $T>2\cdot k$. If $X_t$ is a random variables such that
 \[
 X_t= \begin{cases}
 c\frac{T-t}{T} & \textrm{w.p. $1/4$}\\
 -c\frac{T-t}{T} & \textrm{w.p. $1/4$}\\
 0 & \textrm{w.p. $1/2$}
 \end{cases},
\]
and $I= \{1,2,\ldots, T/k\}$, then 
\[P\left(\left|\frac{1}{\sqrt{T}}\sum_{i\in I} X_i\right|<a \frac{c}{\sqrt{50k}}\right)\le \erf(a)+\sqrt{\frac{50^3 k}{T}},\]
where $\erf(a)=\Phi(a)-\Phi(-a)$ is the error function.
 \end{lemma}
 \begin{proof}
 First, we lower bound $\sum_{i\in I} \sigma_i^2$, and obtain that:
\begin{align*}
\sum_{i\in I} E[|X_i|^2]
&= \frac{c^2}{2}\sum_{i\in I} \left(\frac{T-t}{T}\right)^2\\
&\ge \frac{c^2}{2T^2} \sum_{t=1}^{T/k} \left(T-t\right)^2\\
&=\frac{c^2}{2T^2}\sum_{t=0}^{T/k-1} \left(\left(\frac{k-1}{k}\right)T+t\right)^2\\
&\ge 
\frac{c^2}{2T^2}\max\left\{\frac{(k-1)^2}{k^2}T^2 \frac{T}{k},\sum_{t=0}^{T/k-1} t^2\right\}
\end{align*}
We also have that for $T>2\cdot k$:
\[\sum_{t=0}^{T/k-1}t^2 =\frac{T/k\left(T/k-1\right)\left(2T/k-1\right)}{6}\ge \frac{T^3}{12k^3}\]
Taken together we obtain that
\begin{align*}
    \sum_{i\in I} \EE[|X_i|^2] &\ge \frac{c^2T}{2}\max\left\{ \frac{(k-1)^2}{k^2}, \frac{1}{12 k^2}\right\}\\
    &\ge \frac{c^2T}{50k}
\end{align*}
 Next, we lower bound $\sum \rho_i$:
 \begin{align*}
    \sum_{i\in I} \EE[|X_i|^3]\le\frac{1}{2}\sum_{t=1}^{T/k}\EE[ \big|c \frac{(T-t)}{T}\big|^3]
    =\frac{c^3}{2T^3}\sum_{t=1}^{T/k} (T-t)^3
    &\le \frac{c^3 T}{2T^3 k}T^3\le \frac{c^3 T}{2k}\\ \label{eq:berry_essen_third_moment}
\end{align*}
 
Taken together we obtain that

\begin{align*} P\left(\left|\frac{1}{\sqrt{T}}\sum X_i\right|<a\frac{c}{\sqrt{50k}}\right)&\le 
P\left(\left|\frac{1}{\sqrt{\sum_{i=1}^T \sigma_i^2}}\sum X_i\right|<a\right)\\
&\le \Phi(a)-\Phi(-a) + 2\left(\sum_{i=1}^T \sigma_i^2\right)^{-3/2}\sum_{i=1}^T\rho_i\\
&\le\Phi(a)-\Phi(-a) +  \frac{(50k)^{3/2}c^3 T }{c^3 T^{3/2}k}\\
& \le \Phi(a)-\Phi(-a) + \sqrt{\frac{50^3 k}{T}}
\end{align*}
\end{proof}

\section{Proofs: Distribution Independent Regularizers}
\subsection{Proof of \cref{thm:gdwarmup}}\label{prf:gdwarmup}

As discussed, the main technical gadget behind our distribution-independent-regularization results is a construction of a convex function on which GD does not converge to the minimal norm solution:

\begin{theorem}[GD does not converge to nearest solution]\label{lem:noneuclid}
Let $\cW=\{\w: \|\w\|<5\}$. For every $0<\theta_2\le 1$, and $0<\theta_1 \leq 0.025\,\theta_2$, there exists a 
a non-negative, convex,  $1$--smooth, and $1$--Lipschitz function $F=F_{\theta_1,\theta_2}$ such that, if we run GD (as defined in~\cref{GD_alg})
with step size $0<\eta<1$ over $F$ then GD outputs $\w_F$ that satisfies the following 
\begin{equation}\label{eq:convg}
    \|\w_F-(\theta_1,\theta_2)\| \leq \frac{2640}{\eta T},\end{equation} but \begin{equation}\label{eq:ereq}F((0,\theta_2))= F((\theta_1,\theta_2))=0
    .
\end{equation}
\end{theorem}

In words, even though $(0,\theta_2)$ and $(\theta_1,\theta_2)$ are both minimizers of $F$, GD converges closer to the latter despite it having the larger norm (that is, despite being farther away from the initial point---recall that we assume here that GD is initialized at the origin).

The proof of \cref{lem:noneuclid} is provided at the end of this section and we continue with the proof of \cref{thm:gdwarmup}.
 
 \begin{proof}[Proof of~\cref{thm:gdwarmup}]
For every regularizer $r$ we will choose a distribution $D$ that is concentrated on a single function $F$ (dependent on $r$). Note that in this case, the iterates of SGD are completely equivalent to the iterates of GD with input function $F$. That is, \cref{thm:gdwarmup} in fact holds even for deterministic GD, and we continue with the  analysis assuming we run GD over a fixed function $F$.

We now proceed to choose the function $F$ for a given $\lambda$-strongly convex regularization $r$. Denote $\etwo=(0,1)$ and $\c=(0.024,1)$. Consider the set $\interval{\etwo}{\c} = \{\alpha\etwo+(1-\alpha)\c: 0\le \alpha\le 1\}$, and let \[\w^* = \argmin_{\w\in \interval{\etwo}{\c}} r(\w).\] 
We now want to choose function $F \ge 0$ such that $F(\etwo)=F(\c)=F(\w^*)=0$ and that $\w_F$, the output of GD over $F$, will satisfy the following:
\begin{itemize}
    \item $
    \|\w^*-\w_F\| > 0.01;
$.
\item If $\Pi(\w_{F})$ is the projection of $\w_F$ on $\interval{\etwo}{\c}$ then $
    \|\w_{F}-\Pi(\w_{F})\| \leq 1/(\eta T).
$
\end{itemize}
This will conclude the proof. Indeed, by strong convexity: 
\begin{align*}
    r(\w_F) - r(\w^*) 
    &\ge
    \nabla r(\w^*)^\top(\w_F-\w^*) + \frac{\lambda}{2} \|\w^*-\w_F\|^2
    \tag{$\lambda$-strong convexity}
    \\
    &= 
    \nabla r(\w^*)^\top(\Pi(\w_F)-\w^*) + \nabla r(\w^*)^\top(\w_F-\Pi(\w_F)) +
    \frac{\lambda}{2} \|\w^*-\w_F\|^2
    \\
    &=
    \nabla r(\w^*)^\top(\w_F-\Pi(\w_F)) + \frac{\lambda}{2} \|\w^*-\w_F\|^2
    \tag{$w^*$ minimizes $r$ over $\interval{\etwo}{\c}$}
    \\
    &\ge 
    -\frac{2640}{T\eta} + 0.5 \cdot 10^{-4}\lambda 
    \tag{Lipschitz condition}
    ,
\end{align*}
and for $\eta = \Omega(1/\lambda T)$ we would get $r(\w_S) - r(\w^*) \geq \Theta(\lambda)$ as claimed.

We now demonstrate how to choose an appropriate $F$. 
We will consider two possible cases: $\|\w^*-\etwo\|\geq 0.012$, or $\|\w^*-\c\|\geq0.012$.
\begin{itemize}[wide]
    \item 
First assume that $\|\w^*-\etwo\|>0.012$. We then choose
$$F(\w) = \min_{\v\in \interval{\etwo}{\c}}\frac{1}{5280}  \|\w-\v\|^2.$$
%
which can be seen to be $1$-smooth and $1$ Lipschitz on $
\cW$.
A simple analysis of the update step shows that for $\eta<1$, we have that $\w^{(t+1)} = \sum_{i=0}^{t-1}(1-\frac{\eta}{2640})^i \frac{\eta}{2640} \etwo$.
Hence,
\begin{align*} \|\w_F-\etwo\|
&= \left\|\frac{1}{T}\sum_{t=1}^{T} \wt-{\etwo}\right\|\\
&\le \frac{1}{T} \sum_{t=1}^{T} \|\wt-{\etwo}\|\\
&= \frac{1}{T}\sum_{t=1}^{T} \left| 1-\frac{\eta}{2640}\sum_{i=1}^{t-1} (1-\frac{\eta}{2640})^i \right|\\
&= \frac{1}{T}\sum_{t=1}^{T} (1-\frac{\eta}{2640})^t \tag{$\sum_{i=1}^{t-1} (1-\eta)^i = \frac{1-(1-\eta)^{t}}{\eta}$}\\
& \le \frac{2640}{\eta T}
.
\end{align*}
In particular we have that 
\[
    \|\w_F-\Pi(\w_F)\| \leq \|\w_F-\etwo\| 
    \leq 
    \frac{2640}{\eta T}
    ,
\] 
and by simple geometry, we have also
\begin{align*}
    \|\w^*-\w_F\|
    \geq
    \|\w^*-\etwo\| 
    \geq
    0.01.
\end{align*}
\item 
Next we assume that $\|\w^*-\c\|>0.012$. We now apply \cref{lem:noneuclid} with $\theta_1=0.024$ and $\theta_2=1$ and consider $F=F_{\theta_1,\theta_2}$ as in the theorem's statement. Then, we have that $\|\w_{F}-\c\|<\frac{120}{\eta T}$, and we obtain as before that $\|\w_F-\Pi(\w_F)\| \leq \frac{2640}{\eta T}$ and that $\|\w^*-\w_F\|>0.01$, as required.\qedhere
\end{itemize}
\end{proof}

\subsection{Proof of \cref{lem:noneuclid}}
It will be more convenient to construct a function $F$ that is convex, $4$-smooth and $22$-Lipschitz  such that if we run GD with step-size $0<\eta<1/3$ over $F$ then GD outputs $\w_F$ that satisfies \cref{eq:ereq} and
\[\|\w_F-(\theta_1,\theta_2)\|\le \frac{120}{\eta T}.\]
Then, by re-scaling $F\to \frac{1}{22}F$, and observing that running GD on $F$ with step size $\eta/22$ is equivalent to running GD on $\frac{1}{22}F$ with stepsize $\eta$, we obtain the desired result.

Next, we construct $F$. For $0\le \theta_2 \le 1$ and $0\le \theta_1\le 0.025\cdot \theta_2$ let us define the set: $A_{\theta_1,\theta_2}=\{(\alpha,\thetatwo): 0\le \alpha \le \theta_1\}$. In turn, we define the function $F=F_{\theta_1,\theta_2}(\w)$ to be: 
\begin{align}\label{eq:f_A}
    F(\w)
    = 
    \argmin_{\textbf{v} \in A_{\theta_1,\theta_2}} \left\{ \tfrac{1}{2}(\w-\textbf{v})^\top \Sigma (\w-\textbf{v}) \right\}
    ,
\end{align} 
where
\[
    \Sigma=\begin{pmatrix}
        1 & \tfrac12 \\
        \tfrac12 & 1
    \end{pmatrix}
    .
\] 
We start with showing that $F$ is indeed convex, $4$-smooth and $22$-Lipschitz as required (As discussed at the beginning, then we obtain the desired result by rescaling). 

It is a standard fact that 
a function of the above form is indeed convex (see, e.g., Example 3.1 in \cite{boyd2004convex}).
We will next show that $F$ is also $4$-smooth and $22$-Lipschitz. first, one can show that from \cref{eq:f_A}), that the gradient is given by 
 \begin{equation}\label{eq:thegrad} \nabla F(\w) = \Sigma (\w-\v(\w)),\end{equation}
where we denote $\v(\w)=\arg\min_{v\in A_{\theta_1,\theta_2}}(\w-\v)^\top\Sigma(\w-\v)$.
Next, observe that for any $\w,\w'$ we have $\|\v(\w)-\v(\w')\|^2 \leq (\w-\w')^\top \Sigma (\w-\w') \leq \tfrac32 \|\w-\w'\|^2$ as $\v(\w)$ is the projection of $\w$ onto $A_{\theta_1,\theta_2}$ with respect to the norm $\|x\|^2 = x^\top \Sigma x$, and since projections are contracting distances. Then,
\begin{align*}
    \| \nabla F(\w) - \nabla F(\w') \|
    &\leq
    \|\Sigma\| \big( \|\w-\w'\| + \|\v(\w)-\v(\w')\| \big)
    \\
    &\leq
    \big( \tfrac32 + (\tfrac32)^{3/2} \big) \|\w-\w'\|
    \\
    &\leq
    4\|\w-\w'\|
    .
\end{align*}
Also, since $\v(0)=(0,\theta_2)$. We obtain that $\|\nabla F(0)\|\le \theta_2 \frac{\sqrt{5}}{2}$. Thus from smoothness we also get that for any $\w\in \cW$, we have that $\|\nabla F(\w)\|\le \frac{\sqrt{5}}{2}+20\le 22$. This proves that indeed $F$ is convex, $4$-smooth and $22$-Lipschitz.

To next prove the statement, we begin with the following analysis for trajectory of GD over the function $F_{\theta_1,\theta_2}$.

\begin{lemma} \label{cl:trajectory_analysis} 
Let $\w^{(1)},...,\w^{(T)}$ be the sequence defined by running unprojected GD (i.e., with $\cW= \real^d$) over $F$ with step size $\eta \leq \frac{1}{3}$, starting from $\w^{(1)} = 0$ for $T$ iterations. Then there exist $\frac{1}{2\eta}\leq t_0 \leq \frac{3}{\eta} ,t_0\le t_1\le t_0+ \frac{7}{\eta}$ s.t.:
\begin{align*}
    \textrm{for~$1 \le t\le t_0$:}&\quad
    \w^{(t)} = \big(I-(I-\eta\Sigma)^{t-1}\big) \lboundary&\numberthis\label{eq:wtt0}& \quad
    \textrm{where $\lboundary=(0,\theta_2)$;}\\
    \textrm{for $t_0<t\le t_1$:}& \quad  \w^{(t)}= \begin{pmatrix}
     w^{(t_0)}_1 \\
    \left(1-\frac{3\eta}{4}\right)^{t-t_0}\big(w_2^{(t_0)}-\thetatwo\big)+\thetatwo
    \end{pmatrix};&\numberthis\label{eq:wtt0t1}
    \\
    \textrm{for $t_1<t \le T$:}&\quad
    \w^{(t)} = \big(I-(I-\eta \Sigma)^{t-t_1}\big) \rboundary+ \left(I-\eta \Sigma\right)^{t-t_1} \w^{(t_1)} &\numberthis\label{eq:wtt1}&\quad
    \textrm{where $\rboundary=(\theta_1,\thetatwo)$.}
\end{align*}
\end{lemma}
\cref{cl:trajectory_analysis} is the most technical part of the proof, and follows a careful step-by-step analysis of the trajectory of GD over the function $F$; we defer its proof to later in this section and proceed with the proof of \cref{lem:noneuclid}. We also complement the proof with a ``proof by picture'' and a schematic description of the trajectory in \cref{fig:trajectory3} 

\begin{figure}[t]
    \centering
    \includegraphics[width=0.6\textwidth]{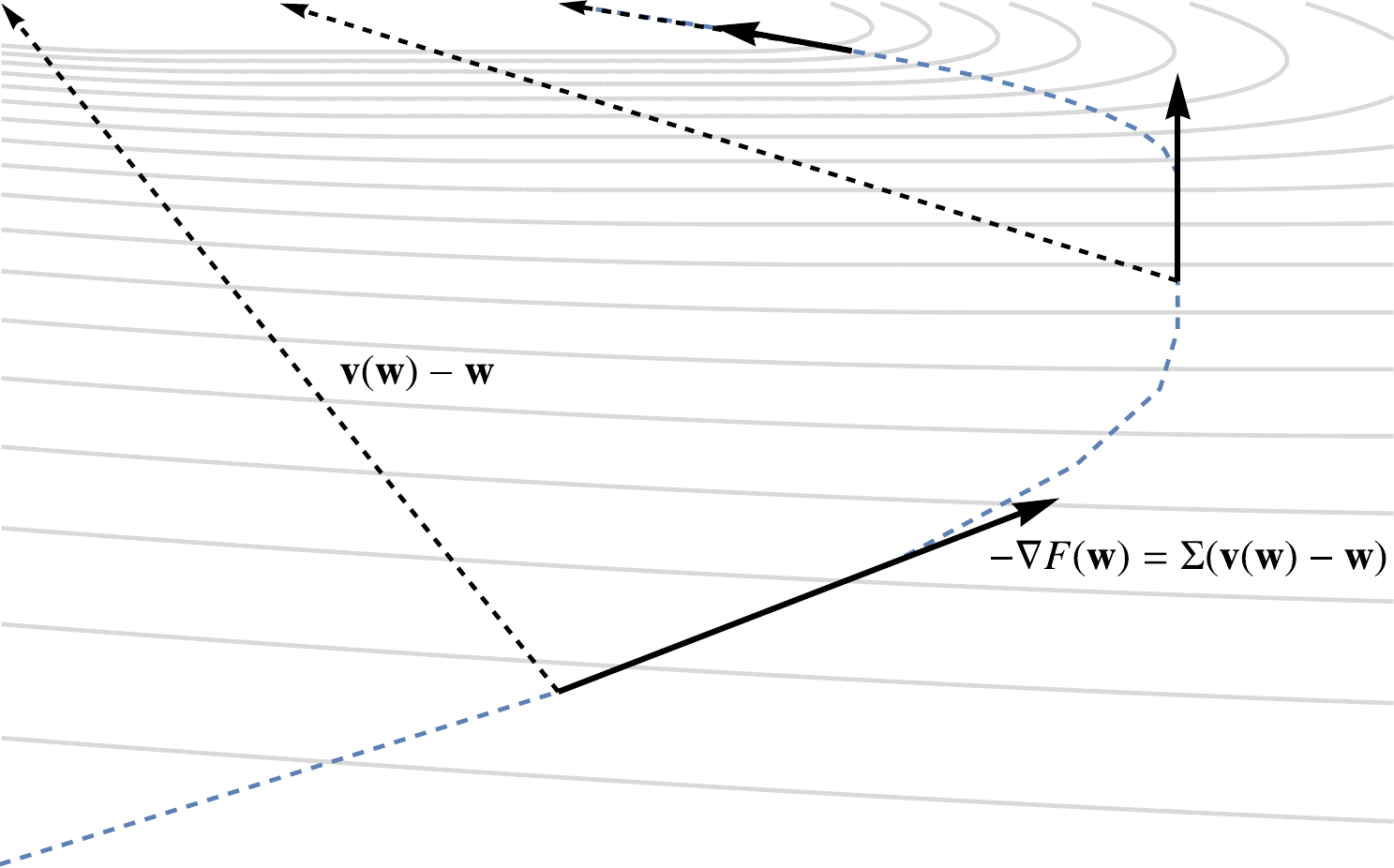}
    \caption{\small We depict here the trajectory being analyzed in \cref{cl:trajectory_analysis}. The trajectory comprises of three phases. In each phase the gradient at point $\w^{(t)}$ is determined by the vector $v(\w^{(t)})$ which is the closest vector to $\w^{(t)}$ on the set $A$, w.r.t the matrix norm induced by $\Sigma$ (see \cref{eq:thegrad}).
     At the first phase, the vector $\v(\w^{(t)})$ is the boundary point $(0,1)$. More generally, the closest point to $\w^{(t)}$ on the interval $\{(\beta,1): -\infty\le \beta\le \infty\}$ is left to $(0,1)$ due to  the linear transformation $\Sigma$. As such $(0,1)$ is the closest point on $A$. The gradient is $\nabla_{\w} =\Sigma (\w-v(\w))$, and points upwards and right.
     This phase continues until $\w^{(t)}$ has moved towards the interior and the closest point $\v(\w^{(t)})$ starts to be at the interior of $A$, then gradient points upwards (note that when the closest point is at the interior of the interval then, horizontally, the distance to the closest point remains constant hence the gradient is vertical).  Finally, the closest point to $\w^{(t)}$ returns to be a boundary point $(\theta_1,1)$ and $\w^{(t)}$ starts to converge towards $(\theta_1,1)$.
    }
    \label{fig:trajectory3}
\end{figure}

\begin{proof}[Proof of \cref{lem:noneuclid}]
We next set out to show that if we run GD on $F$ with any step-size $0<\eta<1/3$, then \[\|\w_{F}-\rboundary\|\le \frac{120}{\eta T}\] and $F(0,\theta_2)=F(\theta_1,\theta_2)=0$. Then, as discussed at the beginning the result follows by rescaling $F$ to obtain a $1$-Lipschitz and smooth function $F$.

We thus proceed with the proof. The fact that $F(0,\theta_2)=F(\theta_1,\theta_2)=0$ is immediate from definitions.

Next, we bound the sizes $\|\lboundary\|, \|\rboundary\|,\|\w^{(t)}\|$ for the setting depicted in \cref{cl:trajectory_analysis}. In particular when $\cW=\real^d$ and no projection steps occur. 
One can easily observe that  $\|\rboundary\|,\|\lboundary\|< 1.5$. Following the trajectory path of $\w^{(t)}$, provided in \cref{cl:trajectory_analysis}, we can also provide a bound on $\w^{(t)}$:
\begin{itemize}
    \item if $t\le t_0$ we have that $\|\w^{(t)}\|<\|\lboundary\|\leq1$;
    \item if $t_0\le t\le t_1$, then $\|\w^{(t)}\|\le \|\w^{(t_0)}\|+ \theta_2 \le 2$;
    \item and if $t\ge t_1$ we have that $\|\w^{(t)}\|\le \|\w^{(t_1)}\|+\|\rboundary\| < 5$.
\end{itemize}
Taken together we have that $\|\w^{(t)}\|<5$. As such, one can show that for any set $\cW$, not necessarily $\cW=\real^d$, as long as $\{\w: \|\w\|\le 5\}\subseteq \cW$ then \cref{cl:trajectory_analysis} holds. Indeed, in any such case running GD or GD without projection is completely equivalent.

Finally, by simple calculation we can show that the singular values of $\Sigma$ are $3/2$ and $1/2$. Hence,
\begin{align}\label{eq:Sigma_spectral}
    \|I-\eta \Sigma\| \leq  1-\frac{\eta}{2}.
\end{align}
where $\|\cdot\|$ denotes the spectral (operator) norm.
We are now ready to show that $\w_{F}$ converges to $\rboundary$:
\begin{align*}
    \|\w_{F}-\rboundary\|
    &\leq \frac{1}{T}\sum_{t=1}^T \|\wt-\rboundary\|
    \\
    & = 
    \frac{1}{T}\sum_{t=1}^{t_1}\|\wt-\rboundary\| +\frac{1}{T}\sum_{t=t_1+1}^{T} \|\wt-\rboundary\|
    \\
    & \leq 
    \frac{10t_1}{T}+\frac{1}{T}\sum_{t=t_1+1}^{T} \|\wt-\rboundary\| \tag{$\|\w^{(t)}\|,\|\rboundary\|<5$} 
    \\
  &=
    \frac{100}{\eta T}+\frac{1}{T}\sum_{t=t_1+1}^T\|(1-\eta\Sigma)^{t-t_1} (\w^{(t_1)}-\rboundary)\| 
    \tag{$t_1<10/\eta$; \cref{eq:wtt1}}
    \\
     & \leq
    \frac{100}{\eta T}
    +\frac{1}{T}\sum_{t=t_1+1}^T\|(I-\eta \Sigma)^{t-t_1}\| \cdot \|\w^{(t_1)}-\rboundary\| \\
    &\le
        \frac{100}{\eta T}
        +\frac{10}{T}\sum_{t=1}^{T-t_1}\|I-\eta \Sigma\|^{t} \tag{$\|\w^{(t_1)}-\rboundary\|<10$}
        \\ 
    &\le
        \frac{100}{\eta T}
        +\frac{10}{T}\sum_{t=0}^{\infty}\Big(1- \frac{\eta}{2}\Big)^{t}  \tag{\cref{eq:Sigma_spectral}}
        \\   
    &\le \frac{100}{\eta T}
    +\frac{10}{T}\cdot \frac{2}{\eta} 
    \leq \frac{120}{\eta T} 
    .
\end{align*}
\end{proof}

\subsubsection{Proof of \cref{cl:trajectory_analysis}}
Computing $\v(\w)$, and from \cref{eq:thegrad} we obtain the following expressions for the gradient

\begin{align}\label{property:derivative}
    \nabla F(\w)=
    \begin{cases}
    \Sigma(\w- (0,\theta_2))& 
    \text{if } \phantom{b_2<}w_1 +\frac{1}{2}(w_2-\theta_2) <0;\\
    \tfrac34 \vectortwo{0}{w_2-\theta}&
    \text{if } 0<w_1 +\frac{1}{2}(w_2-\theta_2) <\theta_1;\\
     \Sigma (\w-(b_2,\theta_2))&
    \text{if } \theta_1 \le w_1 +\frac{1}{2}(w_2-\theta_2) \phantom{\ge b_2}.
    \end{cases}
\end{align}

We thus obtain two boundary conditions that governs the behavior of the trajectory:
\begin{align}
    \label{eq:boundary1}
    w_1 + \tfrac{1}{2} w_2 &< \tfrac{1}{2} \theta_2;
    \\
    \label{eq:boundary2}
    w_1 + \tfrac{1}{2} w_2 &\geq \tfrac{1}{2} \theta_2 + \theta_1.
\end{align}
Given $\eta$, we claim that \cref{cl:trajectory_analysis} holds if we let $t_0$ denote the first iterate such that $\w^{(t_0)}$ violates \cref{eq:boundary1}, when running GD, and if $t_1$ denotes the first iterate for which $\w^{(t)}$ satisfies \cref{eq:boundary2}. We will split the proof into 3 parts, according to GD's trajectory, i.e. $t\leq t_0, t_0< t\leq t_1, t>t_1$.

\begin{claim} \label{cl:before_t0}
There exists $\frac{1}{2\eta}\le t_0 \le \frac{3}{\eta}$ such that $\w^{(t_0)}$ is the first iterate that violates \cref{eq:boundary1}. Further, for any $t\leq t_0$, $\wt$ can be calculated by \cref{eq:wtt0}. And finally, $0.03 \thetatwo\leq \wonet{t_0}$.
\end{claim}

\begin{proof}
First note that $\w^{(1)}$ satisfies \cref{eq:boundary1}, hence $t_0 \geq 1$. Now, following the calculation of the derivative provided in,  \cref{property:derivative} we obtain the update step
 $\w^{(t+1)}=\w^{(t)}-\eta \Sigma (\wt-\lboundary)$ which we can rewrite as
\begin{equation} \label{recursion}
    \w^{(t+1)}=(I-\eta \Sigma )\wt + \eta \Sigma \lboundary.
\end{equation}
By induction one can show that for $2\le t \le t_0$: 
\begin{align*} 
\w^{(t)} 
&=\sum_{i=0}^{t-2} (I-\eta \Sigma)^i \cdot (\eta \Sigma \lboundary) \\
&= \big(I-\left(I-\eta\Sigma\right)^{t-1}\big) \lboundary. \numberthis \label{eq:induction_claim}
\end{align*}
This shows that for any $t\le t_0$, $\w^{(t)}$ can be calculated by \cref{eq:wtt0}. We proceed with the proof to show that $\frac{1}{2\eta}\le t_0\le \frac{3}{\eta}$.
Considering the singular value decomposition of $\Sigma$ one can show that:
\begin{equation}\label{eq:o-sigmat}
(I-\eta \Sigma)^{t-1}=\frac{1}{2} \begin{pmatrix}
(1-\tfrac32\eta)^{t-1}+(1-\tfrac12\eta)^{t-1} & (1-\tfrac32\eta)^{t-1}-(1-\tfrac12\eta)^{t-1} \\
(1-\tfrac32\eta)^{t-1}-(1-\tfrac12\eta)^{t-1} & (1-\tfrac32\eta)^{t-1}+(1-\tfrac12\eta)^{t-1}
\end{pmatrix}.\end{equation}
Plugging this in \cref{eq:induction_claim}, we obtain that for any $t\le t_0$:
\begin{equation} \label{eq:first_trajectory_vector}
    \wt
    =
    \frac{\thetatwo}{2} \begin{pmatrix}
    (1-\tfrac12\eta)^{t-1} - (1-\tfrac32\eta)^{t-1} \\
    2 - (1-\tfrac12\eta)^{t-1} +(1-\tfrac32\eta)^{t-1}
    \end{pmatrix}
    .
\end{equation}
To obtain the lower bound on $t_0$ observe that $t_0$ satisfies: 
$$w^{(t_0)}_1 + \frac{1}{2}(w^{(t_0)}_2-\thetatwo) \geq 0 ,$$
Plugging \cref{eq:first_trajectory_vector} and dividing by $\theta_2/2$ we obtain that:
\[ \left(1-\frac{\eta}{2}\right)^{t_0-1} -\left(1-\frac{3\eta}{2}\right)^{t_0-1}+\frac{1}{2}\left(2- \left(1-\frac{\eta}{2}\right)^{t_0-1} -\left(1-\frac{3\eta}{2}\right)^{t_0-1}-2\right)
\ge 0.\]
Rearranging terms we get:
\[\frac{1}{2}\left(1-\frac{\eta}{2}\right)^{t_0-1}-\frac{3}{2}\left(1-\frac{3\eta}{2}\right)^{t_0-1}\ge 0,\]
which for $\eta<1/3$, can be rewritten as:
\begin{equation}\label{eq:kill_this_paper}\left(1 + \frac{2\eta}{2-3\eta}\right)^{t_0-1}=\left(\frac{2-\eta}{2-3\eta}\right)^{t_0-1} \geq 3.\end{equation}
This leads to
\begin{align*}
    t_0&\ge \frac{1}{\ln (1+\frac{2\eta}{2-3\eta})}
    \tag{$\ln (3) \ge 1$}\\
    &\ge 
    \frac{2-3\eta}{2\eta} 
    \tag{$\ln(x+1)\le x$} \\
    &= \frac{1}{\eta} -\frac{3}{2}\\
    & \ge \frac{1}{2\eta}. 
    \tag{$\eta \le \frac{1}{3}$}
\end{align*}
Next we provide an upper bound for $t_0$. Again, for every $t<t_0$ \cref{eq:boundary1} is satisfied, which, as we already saw (recall \cref{eq:kill_this_paper}) means that for every $t< t_0$:
\begin{equation} \label{eq:upperbound_t0}
    \forall t<t_0, \quad \left(1+\frac{2\eta}{2-3\eta}\right)^{t-1} \leq 3.
\end{equation}
Using the inequality $(1+2/n)^n \ge 3$, we obtain\begin{align*}
 \left(1+ \frac{2\eta}{2-3\eta}\right)^{t-1}
    \ge
    \left( 1+\eta \right)^{t-1}
     \geq 
    3^{\frac{\eta}{2} (t-1)} 
    .
\end{align*}
In particular for $t \ge \frac{2}{\eta}+1$  \cref{eq:upperbound_t0} is violated and hence $t_0 \le \frac{3}{\eta}$. 

Finally, we provide a lower bound for $w_1^{(t_0)}$. Namely, we want to show that $w_1^{(t_0)}\geq0.04\theta_2$.
First, by rearranging terms at \cref{eq:kill_this_paper} we obtain that
 $t_0$ is sufficiently large so that
 $\left(1-\frac{\eta}{2}\right)^{t_0-1}\ge 3\left(1-\frac{3\eta}{2}\right)^{t_0-1}$.
Again applying the formula for $\w^{(t_0)}$ in \cref{eq:first_trajectory_vector} we have that:
\begin{align*}
w_1^{(t_0)}=\frac{\thetatwo}{2}\cdot[(1-\frac{1}{2}\cdot \eta)^{t_0-1} - (1-\frac{3}{2}\cdot\eta)^{t_0-1}] 
&\ge \frac{\theta_2}{4} \left(1-\frac{1}{2}\eta\right)^{t_0-1}\\
& \ge
\frac{\theta_2}{4} \left(1-\frac{1}{2}\eta\right)^{3/\eta} & t_0<\frac{3}{\eta}\\
& \ge 2^{-5}\theta_2. &\left(1-\frac{1}{2n}\right)^n >\frac{1}{2} 
 \numberthis \label{w1t0}
\end{align*}
%
%
This concludes the analysis of the first phase of the trajectory.
\end{proof}

We next move on to the case $t_0\le t\le t_1$.

\begin{claim} \label{cl:between_t0_t1}
Let $t_0 \le t \leq t_1$. Then $\wt$ can be calculated by \cref{eq:wtt0t1}. Moreover $t_1\leq t_0+\frac{7}{\eta}$.
\end{claim}

\begin{proof}
We again apply the calculation of the derivative provided in  \cref{property:derivative} at $t_0 \le t\le t_1$  and obtain : 
\begin{equation} \label{eq:gradient_region2}
    \nabla F(\wt)=\begin{pmatrix}
0 \\ \frac{3}{4}(\wttwo-\thetatwo)
\end{pmatrix}.
\end{equation}
Note that this proves that $w_1^{(t)}= w_1^{(t_0)}$. For $w_2^{(t)}$, we have that
\[\wttwo=w_2^{(t-1)}(1-\frac{3}{4}\eta)+\frac{3}{4} \cdot \eta \cdot \thetatwo,\] which leads by induction to the following:
\begin{align*}\wttwo &=\left(1-\frac{3}{4}\eta\right)^{t-t_0}w_2^{(t_0)}+\sum_{i=0}^{(t-t_0)-1}\left(1-\frac{3}{4}\eta\right)^i \cdot \frac{3\eta\thetatwo}{4}\\
&= \left(1-\frac{3}{4}\eta\right)^{t-t_0}w_2^{(t_0)}+ \left(1-\left(1-\frac{3}{4}\eta\right)^{t-t_0}\right)\theta_2\\
&=\left(1-\frac{3}{4}\eta\right)^{t-t_0}\left(w_2^{(t_0)}-\thetatwo\right)+\thetatwo.
\end{align*}
This shows that for any $t_0\le t\le t_1$ \cref{eq:wtt0t1} holds. 

We next bound $t_1$. Recall that $t_1$ is defined to be the first iterate for which \cref{eq:boundary2} is satisfied. Let us show that for any $t$ s.t $t_0+\frac{7}{\eta}<t$ holds,  \cref{eq:boundary2} is satisfied and hence $t_1\le t_0+7/\eta$. Equivalently we will show that for $t>t_0+7/\eta$, the following equation holds:
\begin{equation} \label{eq:t_1condition}
    \thetatwo-\wtwot{t} \leq    2(\wonet{t}-\theta_1).  
\end{equation}
Indeed, let $t<t_1$, then
\begin{align*}
    2\cdot(\wonet{t}-\theta_1)&= 2\cdot(\wonet{t_0}-\theta_1) & (w_1^{(t)}=w_{1}^{(t_0)} \textrm{~by~} \cref{eq:wtt0t1})  \\
    &\geq 2\cdot (2^{-5}\cdot \thetatwo - \theta_1)&(w_1^{(t_0)}\ge 2^{-5}\theta_2 \textrm{~by~} \cref{w1t0} )\\
    &\geq 2\cdot (2^{-5}\cdot \thetatwo - 0.025\thetatwo)&(\theta_1\le 0.025\cdot \thetatwo )\\
    &\ge 0.01\cdot\thetatwo
\end{align*}
Next assume that $t\ge t_0+\frac{7}{\eta}$, then
\begin{align*}
0.01\cdot\thetatwo & \ge e^{-\frac{3\eta\cdot (t-t_0)}{4}}\thetatwo  & t\ge t_0+\frac{20}{3\eta}\\
    &\ge \left(1-\frac{3}{4}\eta\right)^{t-t_0} \thetatwo \\
    &\ge \left(1-\frac{3}{4}\eta\right)^{t-t_0} \left[\thetatwo-\wtwot{t_0}\right] & (\wtwot{t_0}\geq 0)  \\
    &=\theta_2-\w_2^{(t)}. & \cref{eq:wtt0t1}
\end{align*}
We now move to the last phase of the trajectory.
\end{proof}






\begin{claim} \label{cl:after_t1}
    Let $t \geq t_1$, then $\wt$ can be calculated by \cref{eq:wtt1}.
\end{claim}

\begin{proof}
Let $t\ge t_1$ be such that \cref{eq:boundary2} holds. Then again, we consider the formula of the derivative $\nabla F(\w)$ (see \cref{property:derivative}) and have that
\[ \nabla F(\w) = \Sigma (\w- \rboundary).\]
We obtain the following recursive formula for $t$ if \cref{eq:boundary2} holds for all $t_1\le t' \le t$:
\begin{align*}
    \wt &=(I-\eta \Sigma) \w^{(t-1)}+\eta \Sigma \rboundary \\
    &=(I-\eta \Sigma)^{t-t_1}\w^{(t_1)}+\sum_{i=0}^{t-t_1-1} (I-\eta \Sigma)^i \eta \Sigma \rboundary \\
    &=(I-\eta \Sigma)^{t-t_1}(\w^{(t_1)}-\rboundary)+ \rboundary. \numberthis\label{eq:herenow}
\end{align*}
This shows that $\w^{(t)}$ can be calculated via \cref{eq:wtt1}. It remains thus to show that for any $t\ge t_1$, \cref{eq:boundary2} always holds. We prove this by induction. Note that for the base case, this follows from the definition of $t_1$. We can thus assume by induction hypothesis that $\wt$ satisfies \cref{eq:herenow}, and we want to prove that
\[
    w^{(t)}_1+\tfrac12w^{(t)}_2-\theta_1-\tfrac12\theta_2 \ge 0
    .
\]
For succinctness, let us write 
\[
    \alpha_t= \left(1-\tfrac32\eta\right)^{t-t_1}
    ,\quad \textrm{and}, \quad 
    \beta_t=\left(1-\tfrac12\eta\right)^{t-t_1}
    .
\]
We will denote also $\v=\vectortwo{1}{1/2}$ Then using \cref{eq:o-sigmat} and \cref{eq:herenow} we have that
\begin{align*} 
    w^{(t)}_1+\tfrac12 w^{(t)}_2 -\theta_1 - \tfrac12 \theta_2
    &= \v^\top(\wt-\rboundary)
    \\
    &=\v^\top(1-\eta \Sigma)^{t-t_1}(\w^{(t_1)}-\rboundary) \tag{\cref{eq:herenow}}
    \\
    &=(\tfrac32\alpha_t+\tfrac12\beta_t)(\w^{(t_1)}_1-\theta_1)+(\tfrac32 \alpha_t - \tfrac12 \beta_t)(\w^{(t_1)}_2-\theta_2) \tag{\cref{eq:o-sigmat}}
    \\
    &\ge (\tfrac34\alpha_t - \tfrac34\beta_t)(\w^{(t_1)}_2-\theta_2) \tag{$2(\w_1^{(t_1)}-\theta_1)\ge \theta_2-\w^{(t_1)}_2$} 
    \\
    &\ge 0,
\end{align*}
where the last inequality is true since $\alpha_t\le \beta_t$ for $\eta<1/3$ and we also have that $\w^{(t_1)}_2<\theta_2$.
This concludes the proof of \cref{cl:trajectory_analysis}.
\end{proof}

\subsection{Proof of \cref{thm:gdr}}\label{prf:gdr}

For a vector $\w\in \cW\subseteq \real^2$ let us denote by $\w^\perp:=(w_2,-w_1)$. In particular, we have that $\w^\top \w^\perp =0$ and $\|\w\|=\|\w^\perp\|$. Our proof relies on the following claim which we prove at the end of this section.

\begin{claim}\label{cl:fichs1}
Let $r$ be an admissible regularizer over $\real^2$. There are two points $\w_1$ and $\w_2$ in the unit ball such that for some $-0.005\|\w_1\| <\delta< 0.005 \|\w_1\|$ we have
\[ \w_2 = \w_1 + \delta \w_1^\perp,\]
and $r(\w_1)\ne r(\w_2).$
\end{claim}

We next proceed with proof of \cref{thm:gdr}. 

\begin{proof}[Proof of \cref{thm:gdr}]
Let $\w_1$ and $\w_2$ be as in \cref{cl:fichs1}. First, because GD is invariant to rotations, we can assume w.l.o.g that $\w_1= \|\w_1\|\cdot \etwo$, and hence $\w_2 =(1,\delta )\|\w_1\|$.
We now set $c_r=\frac{1}{2}|r(\w_1)-r(\w_2)|$. To choose $T_r,D_r$ and $\w_r$ we now look at two cases: if $r(\w_1)>r(\w_2)$ and if $r(\w_1)<r(\w_2)$.
\begin{itemize}[wide]
\item 
First suppose $r(\w_1)>r(\w_2)$. By upper-semicontinuity there exists a neighborhood $\delta_1$ such that for every $\w$ s.t. $\|\w-\w_1\|<\delta_1$, satisfies $r(\w)>r(\w_2) + c_r$. We thus set $T_r=\frac{2640}{\eta\delta_1}$, and $\w_r=\w_2$. 
We are left with choosing $D_r$. Note that in this case, the regularizer prefers a point with large Euclidean norm over a point with smaller Euclidean norm. Thus, to show it is not the implicit bias of SGD we only need to construct a distribution that is biased towards smaller Euclidean norms:
Indeed, consider the set $\interval{\w_1}{\w_2}= \{\alpha\w_1+(1-\alpha)\w_2: 0\le \alpha\le 1\}$ we set 
$$f(\w)= \frac{1}{5280}\cdot\min_{\v\in \interval{\w_1}{\w_2}}\|\w-\v\|^2$$
Our distribution $D_r$ is defined to choose $f$ w.p. $1$. Having defined $T_r,c_r,\w_r$ and $D_r$ we now set out to prove the result.
A simple analysis of the update step of SGD shows that for $\eta<1$ we have for every $\wt$ that $\w^{(t+1)} = \sum_{i=0}^{t-1}(1-\frac{\eta}{2640})^i \frac{\eta}{2640} \w_1$. Hence,
\begin{align*}\|\w_S-\w_1\| =\|\frac{1}{T_r}\sum_{t=1}^{T_r} \wt-\w_1\|
&\le \frac{1}{T_r} \sum_{t=1}^{T_r} \|\wt-\w_1\|\\
&= \frac{1}{T_r}\sum_{t=1}^{T_r} \|\sum_{i=1}^{t-1} (1-\frac{\eta}{2640})^i \frac{\eta}{2640} \w_1 - \w_1\|\\
&= \frac{1}{T_r}\sum_{t=1}^{T_r} \| (1-\frac{\eta}{2640})^t \w_1\| & \sum_{i=1}^t (1-\eta)^t = \frac{1-(1-\eta)^{t+1}}{\eta}\\
&\le \frac{1}{T_r}\sum_{t=1}^{T_r} (1-\frac{\eta}{2640})^{t}\\
& \le \frac{2640}{T_r\eta}\\
& = \delta_1 & T_r=\frac{2640}{\delta_1 \eta}
\end{align*}
By property of $\delta_1$ we have that $r(\w_S)>r(\w_2)+c_r$.
But because $\w_2$ is optimal (i.e. attain zero on $f$), we also have $F_S(\w_S)> F(\w_r)$. This proves the case $r(\w_1)>r(\w_2)$.
\item 
Next, assume that $r(\w_1)<r(\w_2)$. As before we have a neighborhood $\delta_2$ such that if $\|\w-\w_2\|<\delta_2$ then we are guaranteed that $r(\w)>r(\w_1)+c_r$. We choose then $T_r=\frac{1}{\delta_2 \eta}$ and $\w_r=\w_1$.
To define $D_r$, we now use the function $F_{\theta_1,\theta_2}$ from \cref{lem:noneuclid}. We assume w.l.o.g that $\delta>0$, if this is not the case we can use that function $F_{\theta_1,\theta_2}(\w)=F_{\theta_1,\theta_2}(-\w)$.
Let us set $\theta_2=\|\w_1\|$ and $\theta_1=|\delta| \|\w_1\| < 0.05\theta_2$. Again, we consider a deterministic distribution $D_r$ that chooses $F_{\theta_1,\theta_2}$ w.p. $1$. 
Recall that we assume that $\w_1=\|\w_1\|\etwo$, hence $\w_1=(0,\theta_2)$ and $\w_2=(\theta_1,\theta_2)$. Hence, by \cref{lem:noneuclid}, if we run over a sample of size $T_r>\frac{1}{\delta_2 \eta}$, we obtain that 
\[\|\w_S - \w_2\|< \delta_2.\]
In particular $r(\w_S)>r(\w_1)+c_r$. But again $F_S(\w_S) \ge F(\w_1)$, because $\w_1$ is optimal.\qedhere
\end{itemize}
\end{proof}

Finally, we prove \cref{cl:fichs1}.

\begin{proof}[Proof of \cref{cl:fichs1}]
First, let us assume that there are $\u,\v$ such that $\|\u\|_2,\|\v\|_2=a$ and $r(\u)\ne r(\v)$ (at the end we will show that for admissible regularizer we always have such two points).
We will also assume that $\|\u-\v\|_2\le 10^{-9}\cdot a^2$. If this was not the case we can cover the sphere $\{\w:\|\w\|_2=a\}$ with balls with radius $10^{-9}\cdot a^2$, and have a constant function at every ball, concluding that $r$ is constant on the sphere (which contradicts our assumption).
Next, we also assume that $|u_2|>\frac{1}{2}a$, (either $|u_1|>\frac{1}{2}a$ or $|u_2|>\frac{1}{2}a$, and the proof is similar in both cases so we will analyse only the later case). 
Then, since $u_{1}^2+u_{2}^2=v_{1}^2+v_{2}^2=a^2$, one can show that
\[ \frac{u_{1}-v_{1}}{v_{2}+u_{2}} = \frac{v_{2}-u_{2}}{u_{1}+v_{1}}.\]
So, by choosing $\delta= \frac{u_{1}-v_{1}}{v_{2}+u_{2}} = \frac{v_{2}-u_{2}}{u_{1}+v_{1}}$ we have that:
\begin{align*}
    |\delta|=\frac{|u_{1}-v_{1}|}{|v_{2}+u_{2}|}
    &=
    \frac{|u_{1}-v_{1}|}{|v_{2}|+|u_{2}|}
    &( |u_2-v_2|<10^{-9}a^2,0.5a<|u_2|)\\
    &\le 
    4\frac{|u_{1}-v_{1}|}{a} &(|v_{2}|+|u_{2}|>a/4)\\
    &\le
    0.0025 a & (|u_{1}-v_{1}|<0.0025/4 \cdot a)
\end{align*}
Using the first equality we can show that $u_{1}-\delta u_{2}=v_{1}+\delta v_{2}$
Similarly we can show that 
$
    u_{2}+\delta u_{1}=v_{2}-\delta v_{1}
.$
Taken together we obtain that $$\u+\delta \u^\perp= \v-\delta \v^\perp.$$ In particular $r(\u+\delta\u^\perp)=r( \v-\delta \v^\perp)$. Since $r(\u)\neq r(\v)$, we either have $r(\u)\ne r(\u+\delta \u^{\perp})$, or $r(\v)\ne r(\v-\delta \v^\perp)$. In the former case we choose $\w_1=\u$, whereas in the latter case we choose $\w_1=\v$.\\

Finally, so far we assume we can find two points on a sphere with different regularization penalty. Next, we assume that on every sphere $r$ is constant. Assume also to the contrary that for every $-0.0025\|\w\|<\delta<0.0025\|\w\|$: \[r(\w + \delta \w^\perp)= r(\w).\] It is not hard to show that in this case $r$ is constant everywhere except maybe $0$, making it in-admissible.
\end{proof}

\section{Proofs II: Distribution Dependent Regularization} 

\subsection{Proof of \cref{lem:r2}}\label{prf:r2}

We start this section by proving the existence of the auxiliary construction in \cref{lem:r2}.

\rtwo*
Before we continue with the proof, notice the following immediate corollary of \cref{lem:r2}:

\begin{corollary}\label{cor:r2}
For every constants $c,\rho>0$, there is a distribution $D$ over a pair of convex functions $\{f(\w;1),f(\w;-1)\}$, such that $f(\w;z)$ is a $\rho$-Lipschitz  convex function in $\real^2$ and, for every $c<\eta<1$ denote $\v_{z,\eta} =- \eta \nabla f(0;z).$ Then the following holds:
\begin{itemize}
    \item For every $z\in\{-1,1\}$ we have that $f(\v_{z,\eta};z)=f(\v_{-z,\eta};z)$;
    \item For every $z\in \{-1,1\}$, $\nabla f(v_{z,\eta},z)=\nabla f(v_{-z,\eta},z)=0$;
    \item $\|\v_{z,\eta}-\v_{-z,\eta}\|>\frac{\rho\eta}{4}$. 
\end{itemize}
\end{corollary}

To derive \cref{cor:r2} from \cref{lem:r2}, take a distribution that w.p. $1/2$ picks $\rho f(\w;1)$ from \cref{lem:r2}, and with probability $1/2$ picks $\rho f(\w;-1)$. One can observe that the result holds.

\begin{proof}[Proof of \cref{lem:r2}]
Let us define $f(\w;\pm 1)$ as follows. Denote $\v_1=-(\tfrac14,\tfrac34)$, $\v_{-1}=-\tfrac34 \etwo$ and let
\begin{align*}
    f(\w;{z})&=
   \max\{ 0,-\v_z^\top \w+c\|\v_{z}\|^2 \}
   .
\end{align*}

It is easy to check that $\nabla f(0;1)=-\v_1$ and that $\nabla f(0;-1)=-\v_{-1}$, and that $\|\v_1-\v_{-1}\|\ge \tfrac14$.
Next, note that if $\eta > c$ then 
\begin{align*}
f(\v_{1,\eta};1)
=\max\big(0,(-\eta+c)\cdot\|\v_{1}\|^2\big)=0=
\max\big(0,-\eta\v_{1}^\top \v_{-1}+c\|\v_{-1}\|^2\big)=f(\v_{-1,\eta};1)
\end{align*}
Similarly, $f(\v_{-1,\eta};-1)=0=f(\v_{1,\eta};-1)$.
Note that, because $f\ge 0$ the above also proves that $\nabla f(v_{z,\eta},z)=\nabla f(v_{-z,\eta},z)=0$.
\end{proof}

\subsection{Proof of \cref{thm:sgdr}}\label{prf:sgdr}

\cref{thm:sgdr} is an immediate corollary of the following theorem:

\begin{theorem}\label{thm:sgdistance}
Let $\cW=\{\w: \|\w\|\le 1\}$. For every $T$ and constant $C>2$, there exists a distribution $D$ over $1$-Lipschitz convex functions over $\real^d$ where $d=10\cdot T$ such that if we run SGD with step size $1/T^2 <\eta \le C/\sqrt{T}$, the following holds:
for any regularizer $r$, w.p.~at least $1/10$ over the sample $S$ there is $\w_r \in \cW$ such that
\begin{align*}
    F_S(\w_r) &\;\leq\; F_S(\w_S)~,
    \\ 
    r(\w_r) &\;\leq\; r(\w_S)~,
\end{align*}
Moreover
\begin{align*}
    \|\w_r-\w_S\|^2_2 &\;\ge \; \frac{T \eta^2}{500C^2}~.
\end{align*}
\end{theorem}
To see how \cref{thm:sgdr} follows, Let $\w^*$ be the minimizer of $r(\w)$ amongst all $\w\in \cW$ with $F_{S}(\w)\le F_{S}(\w_S)$ then by strong convexity
\[r(\w_{S}) \ge r(\w^*)+\frac{\lambda}{2}\|\w_S-\w^*\|^2.\]
Now if $\|\w_S-\w^*\|>\frac{1}{4}\cdot\|\w_r-\w_S\|$ we are done. If not, then 
$$\|\w_S-\w^*\| \leq \frac{1}{4}\|\w_r-\w_S\|\leq\frac{1}{4}[\|\w_r-\w^*\|+\|\w_S-\w^*\|],$$
which leads to
$\|\w_r-\w^*\|\geq\frac{3}{4}\cdot\|\w_r-\w_S\|$. Using this, we get by strong convexity:
\begin{align*}
    r(\w_S)&\ge r(\w_r)
    \ge r(\w^*) + \frac{\lambda}{2}\|\w_r-\w^*\|^2
     \ge r(\w^*) + \frac{9\lambda}{32}\|\w_r-\w_S\|^2.
\end{align*}
Now the result follows from \cref{thm:sgdistance}.

\paragraph{Proof of \cref{thm:sgdistance}}
Choose $d=10\cdot T$. Let $D_0$ be the distribution over convex functions in $\real^2$ whose existence follows from \cref{cor:r2} with $c<1/(4T^2)$ and $\rho=2/C$.

We now define a distribution over convex functions in $\real^d$ as follows: at each iteration pick uniformly $\z$ from the set $\{\z=(z;i): z\in \{-1,1\}, i=1,...,5T\}$ and let:

\[\f(\w;\z)= f((w_{2i-1},w_{2i});z).\]

To prove the result we proceed as follows: given a sample $S$ drawn i.i.d from the distribution $D$, let us call a sample point $\z_t=(z_t,i_t)$ \emph{good} if $t<T/2$ and if $i_t$ appears only once in the sample (i.e. for any $t'\le T$, $i_{t'}\ne i_t$). Denote by $S_g$ the set of good samples. 

Next for a sample $S$ define a sample $S'=\{\z'_1,\ldots, \z'_{T}\}$ to be a sample that differ from $S$ only at good sample points, and for every good sample point if $\z_t=(z_t,i_t)$ then $\z'_t=(z'_t,i_t)=(-z_t,i_t)$. It is not hard to see that $S$ and $S'$ are identically distributed (though dependent).

Now first, we want to show that $F_{S}(\w_S)=F_{S}(\w_{S'})$ w.p. $1$ and that w.p. $0.2$ we have that 

\[\|\w_{S}-\w_{S'}\|> \frac{\sqrt{T}\eta}{22C}.\]
If we can show that, then we are done. Indeed, by symmetry, we have with probability $1/2$ $r(\w_{S'})\le r(\w_{S})$. We can then take $\w_r=\w_{S'}$. Taken together we have that with probability $0.1$ all the requirements of the theorem hold.

\begin{claim}\label{cl:FS=FS'}
 $F_{S}(\w_S)=F_{S}(\w_{S'})$
\end{claim}
\begin{proof}
Fix a sample $S$. To avoid cumbersome notations, and because $S$, $S'$ are fixed, we will denote here $\w_{S}=\bar\w$ and $\w_{S'}=\bar\w'$.
Next, for a vector $\w$ and coordinate $i_t$ let us also denote $\w(i_t)= (w_{2i_t-1},w_{2i_t})\in \real^2$. 

We first analyze the trajectory of SGD over a sequence $\{\z_1,\ldots,\z_t\}$. One can prove, by induction, that at step $t$ the algorithm chooses point $\w^{(t)}$ as follows:
\begin{align}\label{eq:wti}
\w^{(t)}(i)=
\begin{cases}
-\eta \nabla f(0,z_{q}) & \textrm{If $i=i_q$ for some $q\le t-1$ and $q=\arg\min\{q': i_{q'}=i_q\}$}\\
0 & \textrm{else}
\end{cases}
\end{align}
Indeed, for $t=1$ this follows from initialization at $0$. For $t\ge 1$ we have that \begin{align*}  w^{(t+1)}=\Pi_\cW \left(\w^{(t)}-\eta\nabla \f(\w^{(t)},z_t)\right).
\end{align*}
Now first assume that for some $q\le t-1$, we have that $i_t=i_p$, then by assumption we have that $\w^{(t)}(i)= -\eta \nabla f(0,z_{q})= v_{z_q,\eta}$ in the notation of \cref{cor:r2}. Also by \cref{cor:r2} we have that $\nabla \f(\w^{(t)},\z_q)=\nabla f(w^{(t)}(i_t),z_q)=0$.

Next, if no such $p$ exists we have by induction hypothesis that $w^{(t)}(i_t)=0$, the result will now clearly follow if we can show that 
\[ \Pi_{\cW}\left(\w^{(t)}-\eta \nabla \f(\w^{(t)},z_t)\right)= 
\w^{(t)}-\eta \nabla \f(\w^{(t)},z_t).\]
But since $\f(\w,\z_t)$ depends only on the tuple in $i_t$ we have that $w^{(t)} \perp \nabla  \f(\w^{(t)},z_t)=f(0,z_t)$ and we obtain that
\begin{align*}\|\w^{(t)}-\eta \nabla \f(\w^{(t)},z_t)\|^2 &= \|\w^{(t)}\|^2 +\eta^2 \nabla f(\w^{(t)},z_t)\|^2\\& 
= \sum_{k=1}^d \|\w^{(t)}(k)\|^2 +\eta^2     \|\nabla f(0,z_t)\|^2\\&
\le \eta^2 \sum_{q=1}^t \|\nabla f(0,z_q)\|^2 \\&
\le \eta^2 T\rho^2 \\&\le 1.\end{align*}

This proves that \cref{eq:wti} holds.

Next, the value $\f(\w; \z_t)$ depends only on $\w(i_t)$ (i.e. independent of the other coordinates). Also, for any $i$ and $t$, we have that  $\w^{(t)}(i)$ depends only on $\z_{t'}$'s such that $t'\le t$ and $i_{t'}=i$. In particular, for any $\z_t \notin S_g$ we have that $\bar \w(i_t)=\bar\w'(i_t)$, hence

\[ f(\bar\w;\z_t)=f(\bar\w(i_t);z_t)=f(\bar{\w}'(i_t);z_t)=f(\bar\w';\z_t).\]

Next, we want to show that for a good coordinate $\z_t$ we also have that $f(\bar \w;z_t)=f(\bar \w';z_t)$. For this, as in \cref{cor:r2} let us denote for any $\eta$ and $z$ by $\v_{z,\eta}=-\eta\nabla f(\textbf{0};z)\in \real^2$.
 Then, for any good coordinate we can show that 
\begin{align}\bar\w(i_t)&= -\frac{T-t}{T}\eta \nabla f(\textbf{0};z_t)=\v_{z_t,\eta'},\label{eq:onestep}\\
\bar\w'(i_t)&= -\frac{T-t}{T}\eta \nabla f(\textbf{0};-z_t)=\v_{-z_t,\eta'}\label{eq:reflectstep}
\end{align}
where $\eta'=\frac{T-t}{T}\eta>\frac{1}{2}\eta>c$. Indeed, recall that we chose $c=1/(4T^2)$. 
Thus, from \cref{cor:r2} we obtain that $f(\bar\w(i_t),z_t)=f(\bar\w'(i_t),z_t)$ and in particular
\[ \f(\bar\w;\z_t)=\f(\bar\w';\z_t)
.\]
\end{proof}
\begin{claim}\label{cl:sgdistance1} w.p. at least $0.2$ we have that
\[\|\w_{S}-\w_{S'}\|> \frac{\sqrt{T}\eta}{22C},\]
\end{claim}
\begin{proof}
Again we will use the notation $\bar\w=\w_{S}$ and $\bar\w'=\w_{S'}$. Note that by \cref{cor:r2}, as well as \cref{eq:onestep,eq:reflectstep} we have that
\[\|\bar\w(i_t)-\bar\w'(i_t)\| \ge \eta\rho/8,\]
for any good sample point $\z_t$. Now:

\begin{align*}
\|\bar\w'-\bar\w\|^2&=
\sum_{i=1}^{5T} \|\bar\w(i)-\bar\w'(i)\|^2\\
&\ge \sum_{i\in S_g} \|\bar\w(i)-\bar\w'(i)\|^2\\
& \ge \frac{|S_g|(\eta\rho)^2}{64}. \numberthis\label{eq:Sgleft}
\end{align*}
Thus, we only need to show that $\EE[|S_g|]>\frac{T}{5}$. Indeed, since $|S_g|<T/2$, we obtain by Markov's inequality that with probability $0.25$, $|S_g|>T/7$

To show that $\EE[|S_g|]>T/5$, for a sample $S$, let $S_b$ contain all coordinates that collided (i.e. $\z_t$ such that for some $\z_{t'}$ we have that $i_t=i_t'$).\\
In order to calculate $|S_b|$ define $\chi_{t,t'}=I(i_t=i_t')$ for every $t,t' \in [T]$. Note that
$Pr(\chi_{t,t'}=1)=\frac{1}{10T}$ and since there are at most $T(T-1)/2$ such pairs we get $\EE[|S_b|]\le \sum_{t,t'} Pr(\chi_{t,t'}=1) \le  (T-1)/20$. Note that any coordinate $i_t$ with $t<T/4$ that did not 
collide is a good coordinate, hence
\begin{align*}
    \EE[|S_g|]&\ge T/4-\EE[|S_b|]\\
    &\ge T/5
\end{align*}
\end{proof}

\subsection{Proof of \cref{thm:nouc}}\label{prf:nouc}
\cref{thm:nouc} follows from the following refined statement:
\begin{theorem}\label{thm:noucquant}
Let $\cW=\{\w: \|\w\|\le 1\}$. For every $T$ and constant $C>2$, there exists a distribution $D$ over $1$-Lipschitz convex functions over $\real^d$ where $d=10^5\cdot T$ such that if we run SGD with step size $1/T^2 <\eta < C/\sqrt{T}$, the following holds:
for any regularizer $r$, w.p.~at least $1/10$ over the sample $S$ there is a set $\cW_{S} \subseteq \cW$ such that
\begin{align*}
    \sup_{\w\in \cW_S} F_S(\w) &\;\leq\; F_S(\w_S)~,
    \\ 
    \sup_{\w\in \cW_S}  r(\w)&\;\leq\; r(\w_S)~,
\end{align*}
Moreover $\cW_S$ is $(2T,10^{-4}\frac{\sqrt{T}\eta}{C})$ statistically complex.
\end{theorem}
Note that since $\cW_S\subseteq K_{S,r}(\w_S)$ we derive as a corollary \cref{thm:nouc}
\paragraph{Proof of \cref{thm:noucquant}}
Again, let $D_0$ be the distribution from \cref{cor:r2} with $c>\frac{1}{kT^2}$, for some constant $k$ (to be determined later) and $\rho=C/2$.
We define a distribution $D$ over $\real^d$, where we let $d=100T\cdot k$. as follows: pick $k$ r.v $\{z^{(1)},...,z^{(k)}\}\in\{-1,1\}$ and $k$ distinct coordinates  $\{i^{(1)},\ldots,i^{(k)}\}\in [d/2]$ (chosen uniformly from all possible distinct $k$-tuples), set 
\[
    \f(\w;\z)
    =
    \frac{1}{k}\sum_{\ell=1}^k f((w_{2i^{(\ell)}-1},w_{2i^{(\ell)}}),z^{(\ell)}).
\]
Analogously to \cref{thm:sgdistance}, for a given sample $S$, let us define $S_g$ to be the set of ``good samples" as follows: a tuple $(\z_t,\ell)$ is said to be \emph{good} if $t<T/2$ and $i^{\ell}_t$ did not collide. Namely, any other sampled coordinate, $i^{(\ell')}_{t'}$ with $\ell'\in [k]$ and $t'\in [T]$ we have that if $i^{(\ell')}_{t'}= i^{(\ell)}_t$, then $\ell'=\ell$ and $t'=t$.

Next, for every sample $S$ define \[\cS(S)=\{S'=(\z'_1,\ldots, \z'_T): {i'_t}^{(\ell)} = i_t^{(\ell)} \forall t,\ell \textrm{~and}~\forall (\z_t,\ell)\notin S_g~ {z'}^{(\ell)}_t=z^{(\ell)}_t~\}.\]
In words, $\cS(S)$ includes all samples where at a good coordinate $(\z_t,\ell)$, $z_t^{\ell}$ may flip. And let us also define:\[\cW_S=\{\w_{S'}, S'\in \cS(S), r(\w_{S'})\le r(\w_S)\},\]

Analogously to \cref{thm:sgdistance} the statement holds once we prove the following two facts: first we show that for every $\w_{S'}\in \cW_S$ we have that
$F_S(\w_S)=F_{S'}(\w_{S'})$ and secondly, we show that $\cW_{S}$ is $(\frac{Tk}{62},\frac{30^{-3}}{kC}\eta \sqrt{T})$- statistically complex (claims \ref{cl:iwantthistoend} and \ref{lem:cScomplex} respectively). Thus, by setting $k=124$ we obtain the desired result.
\begin{claim}\label{cl:iwantthistoend}
For every $\w_{S'}\in \cW_S$ we have that $F_S(\w_S)=F_{S'}(\w_{S'})$.
\end{claim}
\begin{proof}
The proof is very similar to the analog case in \cref{thm:sgdistance}, and by a similar argument (which we omit) we can show that for every $t$
\[ \w^{(t)}(i)= 
\begin{cases}
-\frac{\eta}{k}\nabla f(0,z^{(j)}_q) & \textrm{if for $q\le t-1$ and $j\le k$ we have $i=i^{(j)}_q$ and for all $q'\le q$ and $j'\in [k]$, $i\ne i^{(j')}_{q'}$.} \\
0 & \textrm{else}
\end{cases}.\]
Fix $S'\in \cS(S)$ and use the shorthand notation $\bar\w$ for $\w_S$ and $\bar \w'$ for $\w_{S'}$ as in the proof of \cref{thm:sgdistance}, we will also use $\w(i,\ell)=(w_{2i^{(\ell)}-1},w_{2i^{(\ell)}})\in \real^2$.
Another notation we add, as in \cref{cor:r2}, is as follows: for $\eta$ and $z$, $\v_{z,\eta}=-\eta \nabla f(0;z)\in \real^2$. 

Then for any sample $(\z_t,\ell)\in S_g$, in $S_g$, we can show that 
\begin{align}\bar\w(i_t,\ell)&= -\frac{T-t+1}{kT}\eta \nabla f(0;z_t)=\v_{z^{(\ell)}_t,\eta_t'},\label{eq:onestepuc}\\
\bar\w'(i_t,\ell)&= -\frac{T-t+1}{kT}\eta \nabla f(0;z_t)=\v_{z'^{(\ell)}_t,\eta_t'},\label{eq:reflectstepuc}
\end{align}
Where $\eta'=\frac{T-t+1}{kT}\eta>c$. Next, for any coordinate $(\z_t,\ell)\notin S_{g}$ we can show that $\bar \w(i,\ell)= \w'(i,\ell)$, hence if $\z_t$ is such that $(i_t,\ell_t)=(i,\ell)$ we clearly have that $f(\bar\w,z^{(\ell)}_t)=f(\bar\w',z^{(\ell)}_t)$. Now for $(\z_t,\ell)\in S_g$, from \cref{cor:r2} we obtain that
\begin{align*}
    \f(\w,z^{(\ell)}_t)
    &=\frac{1}{k}\sum_{\ell=1}^k f(\bar\w(i_t,\ell);z^{(\ell)}_{t})\\
    &=\frac{1}{k}\sum_{\ell=1}^k f(\w_{\eta_t',z_t^{(\ell)}};z^{(\ell)}_{t})&\cref{eq:onestepuc}\\ 
    &=\frac{1}{k}\sum_{\ell=1}^k f(\w_{\eta_t',{z'_t}^{(\ell)}};z^{(\ell)}_{t})&\cref{cor:r2}\\    
    &=\frac{1}{k}\sum_{\ell=1}^k 
    f(\bar\w'(i_t,\ell);z^{(\ell)}_{t})
     &\cref{eq:reflectstepuc}\\
    &=\f(\bar\w',z^{(\ell)}_t) 
\end{align*}
\end{proof}

Next we prove the statistical complexity of $\cW_S$:

\begin{claim}\label{lem:cScomplex}
The set $\cW_S$ is $(\frac{Tk}{62},\frac{\eta \sqrt{T}}{30^3 k})$--statistically complex.
\end{claim}
\begin{proof}
One can show that if we randomly pick $S$ and then pick uniformly an elements from $S'\in \cS(S)$ then $S$ and $S'$ are identically distributed. As a corollary if we pick a random sample $S$ then w.p. 0.5 we have that 
\[|\cW_S|\ge \frac{|\cS(S)|}{2}.\]


We next argue that any set $A\subseteq \{\w_{S'}: S'\in \cS(S)\}$ such that $|A|>\frac{|\cS(S)|}{2}$, then $A$ is $(\frac{Tk}{62},\frac{30^{-3}}{k}\eta \sqrt{T})$- statistically complex

Indeed, fix $S$. Similar to the argument in \cref{cl:sgdistance1}, we have that with probability $0.2$, that $|S_g|> T\cdot k/7$. We claim that if this event occurred then every subset of size  $|\cS(S)|/2$ will be statistically complex.

Indeed, let us index the coordinates of $\real^{|S_g|}$ by the elements of $S_g$. Then, for every element $\w\in \{\w_{S'}:S' \in \cS\}$ we let  $\u(\w):\real^{d}\to \real^{|S_g|}$ be an affine projection such that: if $(\z_t,\ell)\in S_g$, then $\u(w)_{(\z_t,\ell)}$ satisfies the following:

\[
\u(w)_{(\z_t,\ell)}=
\begin{cases}
\frac{1}{\sqrt{|S_g|}} & \w(i_t,\ell)=\v_{1,\eta_t'}\\
-\frac{1}{\sqrt{|S_g|}} &\w(i_t,\ell)=\v_{-1,\eta_t'}\end{cases}
\]
It can be seen from \cref{eq:onestepuc} and \cref{eq:reflectstepuc} and \cref{cor:r2} that $\|\v_{1,\eta'_t}-\v_{-1,\eta'_t}\|>\frac{T-t+1}{kT}\eta\rho/4>\frac{\eta\rho}{12k}$, hence we can define $\u$ to be $g$-Lipschitz where
\[g= \frac{24 k}{\eta\rho\sqrt{T}}.\]

Combining this with \cref{cor:feldman}, we get that there exists a distribution $D$ over $1$-Lipschitz convex functions such that, given $m=|S_g|/6>Tk/62$ elements from $D$, with probability $1/4$ there is $\w\in A$ such that 
\[\frac{1}{m}\sum_{i=1}^m f(\w,z_t)=0.\]
but,
\[\EE_{\z\sim D}f(\w,z)>3/(g\cdot 4)>0.003\frac{\rho\eta \sqrt{T}}{k}=0.003\frac{\eta \sqrt{T}}{kC}\]

\end{proof}

\section{Proof of  \cref{thm:nonconvex}}\label{prf:noconvex}
We begin the construction by the definition of the distribution $D$:\\
\begin{align*}
&f(\w;z=1)=\begin{cases}
w_1 & \text{if } \w \in A \\
0 & \text{else}
\end{cases}, & f(\w;z=3)=w_2, \\
& f(\w;z=2)=\begin{cases}
-w_1 & \text{if } \w \in A \\
0 & \text{else}
\end{cases},
& f(\w;z=4)=-w_2
\end{align*}

Where $z \sim Uniform([1,2,3,4])$ and 

\[
A=\{(w_1,w_2): |w_1| ,|w_2| \leq \frac{1}{4}
\}.
\] 

Note that by symmetry $F=E_z[f(\w, z)]=0$, and indeed in expectation this is a convex function.

For the proof we will define two ``good" events, set $c=\eta \sqrt{T}=\Theta(1)$, and let:
\begin{align*}
    &E_1: |\wstwo|> \frac{1}{4}
    , 
    \\
    &E_2(\beta): |\wsone|>\frac{\eta \sqrt{T}}{2}\cdot \beta \\
\end{align*}
where we write $\w_S= (w_1^S,w_2^S)$, and $\beta$ is a parameter sufficiently small so that. 
\[\erf(\beta)\le \sqrt{\erf\left(\frac{\sqrt{50}}{4c}\right)}- \erf\left(\frac{\sqrt{50}}{4c}\right).\]
Note that $\beta$ depends only on $c=\Theta(1)$.

Let us denote by $E(\beta)=E_1\cap E_2(\beta)$, then we will rely on the following claim that lower bounds the probability of the event $E$. We deter the proof of the claim to the end of the section and continue with the proof:
\begin{claim}\label{cl:good}
Let $E(\beta)=E_1\cap E_2(\beta)$ and suppose that $z\sim D$ then, for our choice of $\beta$, and sufficiently large $T$ 
\[P(E(\beta))>1-\sqrt{\erf\left(\frac{50}{4c}\right)}.\]
\end{claim}

\begin{proof}
Taking \cref{cl:good} into account, Fix a random sample $S$. Let $\beta$ and $T$ be as in \cref{cl:good} and assume that event $E:=E(\beta)$ occurred. Throughout, let us denote $c=\eta \sqrt{T}$.

To show that the statement holds, we define $\w^*_0=(0,\bar w_2^S)$ and $\w^*_{-1}=(-w^S_1,\bar w^S_2)$. We will show that for one of these candidate vectors the statement holds.

First we want to show that if $\eta =\Theta(1/\sqrt{T})$, then $\|\w_S-\w^*_0\|=\Theta(1)$. Indeed, note that since $E_2(\beta)$ occurred $$\|\ws-\w^*_0\|_2\ge |\wsone|\ge\frac{\eta \sqrt{T}}{2}\cdot \beta=\Theta(1).$$ 
Similarly $\|\w_S-\w^*_{-1}\|=\Theta(1)$.

Next we want to show that $F_S(\w^*_0)\le F(\bar\w)$, or $F_S(\w^*_{-1})\le F(\bar\w)$. Note that for every $\w$ such that $|w_2|\ge \frac{1}{4}$, for every $z=\{1,2,3,4\}$, $f(\w;z)$ depends only on the second coordinate, namely $w_2$. In particular, if $|w^S_2|\ge \frac{1}{4}$ we obtain by the construction that $F_S(\w_S)=F_S(\w^*_0)=F_S(\w^*_{-1})$. Thus, due to event $E_1$ we obtain the desired result.

Finally, we want to show $\min \{r(\w^*_0),r(\w^*_{-1})\}<r(\w_S)$, w.p probability at least $1/4$. First, assume that with probability $1/2$ we have that $r(\w^*_{-1})\ne r(\w_S)$. By symmetry one can show that in this case we have that $r(\w^*_{-1})<r(\w_S)$ with probability $1/2$. Next, assume that $r(\w^*_{-1})=r(\w_S)$ with probability at least $1/2$. In this case, we obtain that:
\begin{align*}
    r(\w^*_0)&=r(0.5\cdot \ws+0.5\cdot \w^*_{-1})\\
    &<\max(r(\ws),r(\w^*_{-1}))\\
    &=r(\ws)
\end{align*}
\end{proof}

We are left with proving \cref{cl:good}

\paragraph{Proof of \cref{cl:good}}

We will bound each event $E_1,E_2$ separately. 
We begin by bounding the event $E_1$:
\paragraph{Bounding $E_1$:}
For $E_1$ we claim the following:
\begin{equation}\label{eq:E1}
Pr\Big(|\wstwo|\leq \frac{1}{4} \Big) \leq \erf\left(\frac{\sqrt{50}}{4c}\right)+\sqrt{\frac{50^3}{T}}\end{equation}
where $\Phi$ is the CDF of a mean zero unit variate normally distributed random variable, and $\erf$ is the error function, namely $\erf\left(x\right)=1-2\Phi(-x)$.

Note that if $\eta =O(\frac{1}{\sqrt{T}})$, given the above bound, the probability that $|w_2^S| > \frac{1}{4}$ is a constant.

\begin{proof}
Recall that \[w_2^S= \frac{1}{T}\sum \eta (T-t) \frac{\partial f(\wt,z_t)}{\partial w_2},\] and one can observe that $\frac{\partial f(\wt,z_t)}{\partial w_2}$ equals $1$ w.p. $1/4$, $-1$, w.p $1/4$ and $0$ w.p $1/2$, independently of $z_{t'}$ for $t'\ne t$.

Hence, applying \cref{cor:be}, with $c=\eta \sqrt{T}$, $k=1$ and $a=\frac{\sqrt{50}}{4c}$ we obtain that 
\begin{align*}
    P(-\frac{1}{4}\le w_2^S\le \frac{1}{4})=
    &P\Big(-\frac{\sqrt{50}}{4c}\frac{c}{\sqrt{50}}\leq \wstwo \leq \frac{\sqrt{50}}{4c}\cdot \frac{c}{\sqrt{50}}\Big)\\
    &\leq \erf\left(\frac{\sqrt{50}}{4c}\right)+\sqrt{\frac{50^3}{T}} \numberthis\label{eq:e1}
\end{align*}\end{proof}
We next move on to bound $E_2$
\paragraph{Bounding $E_2$:}

Let us consider a random sample $S'=\{z'_1,\ldots, z'_T\}$ that is generated by picking a random sample $S=z_1,\ldots, z_T$ i.i.d distributed according to $D$, and then for every $z_t$ such that $z_t\in \{1,2\}$ with probability half we let $z'_t=1$ and with probability half we let $z'_t=2$. It can be seen that $S'$ is an i.i.d sequence drawn according to the distribution $D$.

Next, let us denote $c=\eta \sqrt{T}$, and a parameter $\alpha$ (to be chosen later). Define the event $$E_\tau:\{S: \min\{t:\wt \notin A\}> \frac{T}{\alpha\cdot c}\}.$$
For our choice of $\beta>0$ we claim that for every $\alpha>0$
$$Pr\Big(|w_1^{S'}|<\frac{\beta }{\sqrt{50\alpha}}\Big| S,S'\in E_\tau\Big)\le 2\erf(\beta)+2\sqrt{\frac{50^2 c^3\alpha}{T}}$$
Indeed, Given $S$, let $\tau = \min\{t: \w^{(t+1)}\notin A\}$ and set $S'_{\tau}=\{z'_1,\ldots, z'_{\tau}\}$ and denote

\[X_{\tau}=\frac{1}{\sqrt{T}}\sum_{t=1}^\tau c\frac{T-t}{T}x_t.\]
where $x_t$ are i.i.d random variables such that w.p. $1/4$ equals $1$, w.p. $1/4$ equals $-1$ and w.p. $1/2$ equals $0$.
Due to symmetry we have that:
\[Pr\Big(|w_1^{S'}|<\frac{\beta }{\sqrt{\alpha\cdot  c}}\Big| S,S'\in E_\tau\Big)\le 2Pr\Big(|X_\tau|<\frac{\beta }{\sqrt{\alpha \cdot c}}\Big| S,S'\in E_\tau\Big)\]

One can observe that \[X_\tau= \sum_{t\in I} \eta \frac{T-t}{T}x_t= \frac{1}{\sqrt{T}} \sum_{t\in I} c \frac{T-t}{T}x_t.\]

Thus applying again \cref{cor:be} with $c=\sqrt{T}\eta$, $I=\{1,\ldots, T/k\}$ with, $k=\alpha\cdot c^3/50$ and $a=\beta$,
we obtain the desired result.

Next, we want to bound $P(E_{\tau})$. Now assume that for some $t<T/(\alpha\cdot c)$, we have that $\wt\notin A$. 

Let $T_{\alpha}=T/(\alpha\cdot c)$ and let $Z_1,\ldots, Z_{T_\alpha}$, be i.i.d copies of a random variable such that $P(Z_t=1)=P(Z_t=-1)=1/2$. Then 
\begin{align*}
P(\neg E_\tau)&\leq 
4P\left(\min\{t: \eta \sum_{i=1}^t Z_i>\frac{1}{4}\}<\frac{T}{\alpha\cdot c}\right) \\
& \le 4P\left(\min\{t: \eta \sum_{i=1}^t Z_i>\frac{1}{4}\}<T_\alpha, \eta\sum_{i=1}^{T_\alpha} Z_i\geq\frac{1}{4}\right)+ 
4P\left(\min\{t: \eta \sum_{i=1}^t Z_i>\frac{1}{4}\}<T_\alpha, \eta\sum_{i=1}^{T_\alpha} Z_i\leq\frac{1}{4}\right)\\
&= 8P\left(\eta \sum_{i=1}^{T_{\alpha}}Z_i\geq\frac{1}{4} \right) \numberthis \label{eq:toZ}
\end{align*}
where the last inequality is by symmetry (reflection principle).
Next, by applying Hoeffding's inequality we obtain that
\begin{align*}
P(\eta \sum_{t=1}^{T_\alpha} Z_t \geq \frac{1}{4})=
P(\frac{\alpha\eta}{\sqrt{T}}\sum_{t=1}^{T_\alpha} Z_t \geq \frac{\alpha}{4\sqrt{T}})
=P(\frac{\alpha c}{T}\sum_{t=1}^{T_\alpha} Z_t \geq \frac{\alpha }{4\sqrt{T}})
\le e^{-\frac{\alpha c}{32}}
\end{align*}
Taken together we obtain that
\[P(\neg~ E_\tau)\leq 8e^{-\frac{\alpha c}{32}},\]
and
\begin{align*}P(\neg E_2(\beta))&\le P(\neg E_2|E_\tau)P(E_\tau)+ P(\neg E_\tau)\\
&\le 
\EE_{S} \left[P(|w_1^{S'}|<\frac{\beta}{\sqrt{\alpha c}}|S,S'\in E_{\tau})\right] + P(\neg E_2)\\
&\le 2\erf(\beta)+2\sqrt{\frac{50^2c^3\alpha}{T}}+8e^{-\frac{\alpha c}{32}}\numberthis\label{eq:e2}
\end{align*}
which yields the desired result.
\paragraph{Bounding $E(\beta)$:}
\cref{eq:e1,eq:e2} yields then:
\begin{align*}
P(\neg E)&< P(\neg E_1)+ P(\neg E_2(\beta))  \\
&\le \erf\left(\frac{\sqrt{50}}{4c}\right)
+2\erf(\beta) + 3\sqrt{\frac{50^2(50+c^3\alpha)}{T}}+8e^{-\frac{\alpha \cdot c}{32}}
\end{align*}
Choosing $\beta$ sufficiently small, one can see that for large enough $\alpha$ and $T$ we obtain the desired result.

\end{document}